\pdfoutput=1
\documentclass{article}
\usepackage{iclr2027_conference,times}
\usepackage[T1]{fontenc}
\usepackage{amsmath,amssymb,mathtools,bm,amsthm}
\newtheorem{proposition}{Proposition}
\usepackage{graphicx,booktabs,tabularx,array,multirow}
\usepackage{xcolor,tikz}
\usetikzlibrary{
  arrows.meta,
  positioning,
  calc,
  fit,
  backgrounds,
  decorations.markings
}
\usepackage{algorithm,algpseudocode}
\usepackage{microtype,placeins,capt-of,wrapfig}
\usepackage{hyperref,url}

\usepackage{munite-title}

\munitesetstyle{chromatic_brackets}

\hypersetup{
  hidelinks,
  pdftitle={MUNITE: Unified Multimodal Latent Inference for Any-to-Any Multimodal Generation},
  pdfauthor={Kyeongmin Yeo, Minhyuk Sung},
  pdfkeywords={any-to-any multimodal generation, latent variable models, flow matching},
}
\newcommand{\headpre}{\vspace{0pt}}
\newcommand{\headpost}{\vspace{0pt}}
\newcolumntype{Y}{>{\raggedright\arraybackslash}X}
\definecolor{latentblue}{RGB}{39,78,114}
\definecolor{latentfill}{RGB}{235,241,247}
\definecolor{inkgray}{gray}{0.35}

\newcommand{\sg}{\operatorname{sg}}

\title{\MUNITEtitle: Unified Multimodal Latent Inference\\for Any-to-Any Multimodal Generation}
\author{Kyeongmin Yeo \quad Minhyuk Sung \\
KAIST \\
\texttt{\{aaaaa, mhsung\}@kaist.ac.kr} 
}

\iclrfinalcopy
\begin{document}
\maketitle
\lhead{}
\begin{abstract}
We introduce MUNITE, a latent-variable framework for flexible any-to-any multimodal generation that treats encoding and latent generation as the same inference problem under different amounts of observed evidence. Given any subset of modalities, MUNITE models the conditional distribution over the latent representation associated with the complete observation. Full observation recovers deterministic encoding, no observation recovers the latent marginal, and intermediate subsets define conditional latent inference, all within a single conditional flow model. A shared latent sample captures variation that must remain consistent across generated targets, while modality-specific generative decoders model the remaining uncertainty independently. To learn these conditional distributions from incomplete training examples, we extend conditional flow matching through self-distillation: predictions conditioned on richer available observations supervise the same model conditioned on smaller subsets at the same intermediate latent state. When the richer-evidence trajectory follows the exact conditional flow, this provides the same expected learning signal as full-target denoising.
Across PolyMNIST-D-Q, FFHQ64, and image--text--audio, MUNITE achieves competitive or better generation quality and source--target alignment, with higher joint-generation coherence. In particular, it attains the highest coherence in all one-to-many and unconditional image--text--audio comparisons, showing the effectiveness of unified latent inference across diverse multimodal settings. Project~page: \href{https://munite-proj.github.io/}{\texttt{munite-proj.github.io}}.
\end{abstract}
\headpre
\section{Introduction}
\headpost

Amid rapid progress in large foundation models for multimodal generation, an important complementary direction is to develop flexible any-to-any multimodal generative models without requiring a large language model as the multimodal backbone.
Many practical applications may depend on only a subset of the capabilities offered by general-purpose language models, often without advanced text reasoning, coding, or mathematical problem solving, and in some cases with little or no text processing.
Such applications call for models specialized to the modalities at hand that still support arbitrary combinations of observed and generated modalities.

A natural alternative is to build multimodal generation around a shared latent space. Latent space modeling has long been successful for images and videos~\citep{rombach2022ldm,blattmann2023svd}, has recently been explored for language generation through diffusion models~\citep{lovelace2023latent}, and is now emerging as a promising basis for any-to-any multimodal generation~\citep{yeo2026muni,cha2026flowbind}. Recent approaches further blur the boundary between representation learning and generation by jointly learning modality encoders, decoders, and the latent generative model, rather than fixing the latent representation before training the generator~\citep{heek2026ul,yeo2026muni}. This suggests a broader possibility: instead of treating multimodal encoding and latent generation as separate components, can they be viewed as instances of the same underlying modeling problem?

We answer this question by \textbf{unifying these roles within a single conditional model}. Our key observation is that full-observation encoding, partial-observation latent inference, and unconditional latent generation correspond to different amounts of evidence about the same latent representation (Figure~\ref{fig:latent-inference}). Given an arbitrary subset of modalities, our model learns the conditional distribution over latent representations of complete observations. When all modalities are observed, this distribution collapses to a point mass, recovering deterministic multimodal encoding. When no modality is observed, it becomes the marginal latent distribution, recovering unconditional generation. Every intermediate subset corresponds to conditional latent inference, so encoding, conditional inference, and unconditional generation form a spectrum indexed by the observed modality subset. This view generalizes the Generative Encoder of UNITE~\citep{duggal2026unite}, whose single network covers only the two endpoints of full and absent observation, to every level of partial observation. The inference model is coupled with modality-specific generative decoders through a shared latent: one latent sample captures the variation that must be shared across generated targets, while the decoders model the remaining modality-specific uncertainty.

\begin{wrapfigure}{R}{0.49\linewidth}
\vspace{-12pt}
\centering

\input{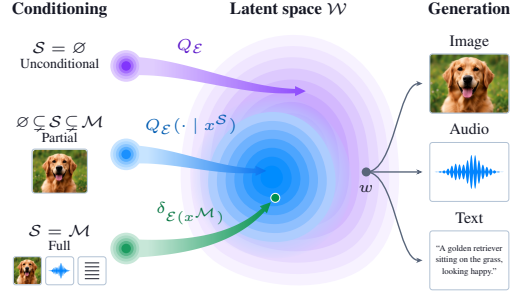}

\caption{%
  \textbf{Latent inference and decoding.}
  Left: latent distributions under different observation subsets.
  Right: factorized decoding from a latent value.%
}
\label{fig:latent-inference}
\vspace{-10pt}
\end{wrapfigure}

This formulation also creates a natural mechanism for self-distillation across observation levels. A partially observed training example does not provide the latent representation of its complete observation as a direct training target. Instead, the model conditioned on a richer set of available modalities provides a clean-latent target for the same model conditioned on a smaller subset, evaluated at the same intermediate flow state. Under the assumptions developed in Section~\ref{sec:distillation}, we show that this self-distillation objective yields the same expected gradient as denoising against the unavailable full-observation latent. This allows partially observed examples to contribute directly to learning the desired conditional distribution.

We evaluate MUNITE on PolyMNIST-D-Q, FFHQ64, and image--text--audio generation, which together span categorical attributes, dense visual modalities, and heterogeneous text, image, and audio, with training data that mix fully and partially observed examples. MUNITE attains the best unconditional fidelity and coherence on PolyMNIST-D-Q and FFHQ64 and the highest coherence in every one-to-many and unconditional image--text--audio comparison, while remaining competitive on one-to-one and many-to-one generation. These results show that a single conditional latent-inference model can support the full range of observation regimes in any-to-any generation, improving joint coherence while retaining the benefits of modality-specific generative decoders.

\headpre
\section{Method}
\label{sec:method}
\headpost

MUNITE learns any-to-any generation from a mixture of fully and partially observed multimodal data. We first introduce a single latent inference model that, for any observed subset of modalities, captures the conditional distribution over latent representations of complete observations. We pair it with modality-specific generative decoders, so that the shared latent carries cross-modal dependencies and the decoders model the remaining uncertainty (Section~\ref{sec:shared-latent}).
We learn the latent representation through cross-modal reconstruction (Section~\ref{sec:reconstruction}) and its evidence-conditioned distributions through self-distillation, which extends conditional flow matching to partially observed examples (Section~\ref{sec:distillation}). An auxiliary contrastive loss further encourages alignment across observation subsets (Section~\ref{sec:contrastive}).

\headpre
\subsection{Shared-Latent Generative Framework}
\label{sec:shared-latent}
\headpost

Let $\mathcal M=\{1,\ldots,M\}$ index the modalities and $X^{\mathcal M}\sim P_{\mathrm{data}}$ denote a complete observation.
We learn a latent representation $W$ of complete observations together with a single model $\mathcal E$ that operates on an arbitrary observed subset $\mathcal S\subseteq\mathcal M$.
A partial observation $x^{\mathcal S}$ generally admits multiple possible completions $x^{\mathcal M}$, which in turn induce a conditional distribution $Q_{\mathcal E}(\cdot\mid x^{\mathcal S})$ over their latent representations.

At the two extremes, the same model $\mathcal E$ reduces to familiar operations.
When $\mathcal S=\mathcal M$, $\mathcal E$ acts as a deterministic encoder, yielding a unique latent representation $w=\mathcal E(x^{\mathcal M})\in\mathbb R^D$. Accordingly, the conditional distribution collapses to $Q_{\mathcal E}(\cdot\mid x^{\mathcal M})=\delta_w$.
When $\mathcal S=\emptyset$, the model instead generates from the latent marginal $Q_{\mathcal E}(\mathrm{d}w)$.
For an intermediate subset $\emptyset\subsetneq\mathcal S\subsetneq\mathcal M$, it models $Q_{\mathcal E}(\mathrm{d}w\mid x^{\mathcal S})$ (Figure~\ref{fig:latent-inference}, left).
We parameterize $\mathcal E$ as a conditional flow model that transports a standard Gaussian to the corresponding distribution $Q_{\mathcal E}(\cdot\mid x^{\mathcal S})$ for any observed subset $\mathcal S$.

\paragraph{A generative-encoder perspective.}
Motivated by the success of latent diffusion models~\citep{rombach2022ldm}, Unified Latents~\citep{heek2026ul} proposed jointly learning latent representations of the data and a generative model of their distribution.
UNITE~\citep{duggal2026unite} took this idea further by introducing a \emph{Generative Encoder} that uses the same model to encode an observed image and, when no image is observed, to generate samples from the latent distribution.
Multimodal data add a natural intermediate regime in which only a subset of modalities is observed, and MUNITE extends the generative-encoder view to it by modeling the conditional distribution of the full-observation latent given an arbitrary observed subset.
Full-observation encoding and unconditional latent generation then become the endpoints $\mathcal S=\mathcal M$ and $\mathcal S=\emptyset$ of a family of latent-inference problems $\{Q_{\mathcal E}(\cdot\mid x^{\mathcal S})\}_{\mathcal S\subseteq\mathcal M}$.

To generate multimodal observations, we pair the latent inference
model $\mathcal E$ with modality-specific generative decoders
$\{\mathcal D_{\psi_m}\}_{m\in\mathcal M}$.
As in multimodal VAEs~\citep{wu2018mvae,shi2019mmvae}, the decoders
generate independently conditioned on a shared latent.
For this factorization to preserve the dependencies in the data,
we require $W$ to be
\emph{coherence-sufficient}~\citep{yeo2026muni}:
\begin{equation}
P_{\mathrm{data}}(\mathrm{d}x^{\mathcal M}\mid W=w)
=
\prod_{m\in\mathcal M}
P_{\mathrm{data}}(\mathrm{d}x^m\mid W=w).
\label{eq:sufficiency}
\end{equation}
The latent thus accounts for cross-modal dependence, leaving
the remaining modality-specific variation to the decoders.
For disjoint input and target sets
$\mathcal S,\mathcal T\subseteq\mathcal M$, marginalizing over
the latent gives
\begin{equation}
P_{\mathrm{data}}(\mathrm{d}x^{\mathcal T}\mid x^{\mathcal S})
=
\int
\prod_{m\in\mathcal T}
P_{\mathrm{data}}(\mathrm{d}x^m\mid W=w)\,
Q_{\mathcal E}(\mathrm{d}w\mid x^{\mathcal S}).
\label{eq:conditional-joint}
\end{equation}
We train each decoder to model the conditional distribution
of its modality given $W$.
At inference, we use $\mathcal E$ to sample one latent $w$
conditioned on $x^{\mathcal S}$ and generate the target
modalities in parallel with
$\{\mathcal D_{\psi_m}(w)\}_{m\in\mathcal T}$.

\headpre
\subsection{Learning the Shared Representation}
\label{sec:reconstruction}
\headpost

We learn the shared representation by reconstructing observed modalities through their generative decoders.
We write $\mathcal E_\theta(w_t,t;x^{\mathcal S})$ for the clean-latent prediction of the model $\mathcal E$ with parameters $\theta$.
For a training example $x^{\mathcal A}$, let the nonempty set $\mathcal A\subseteq\mathcal M$ denote its available modalities.
For reconstruction, we sample a target modality $m\in\mathcal A$ and a conditioning subset $\mathcal S\subseteq\mathcal A$ as described in Appendix~\ref{app:subset-sampling}.
Given a Gaussian query $\epsilon\sim\mathcal N(0,I_D)$ at $t=0$, the reconstruction objective is

\begin{equation}
\mathcal L_{\mathrm{rec}}
=\mathbb E_{(\mathcal A,x^{\mathcal A}),\mathcal S,m,\epsilon}
\left[
\ell_m\left(
\mathcal D_{\psi_m}\big(
\mathcal E_\theta(\epsilon,0;x^{\mathcal S})
\big),x^m
\right)
\right].
\label{eq:reconstruction}
\end{equation}

Here $\ell_m$ denotes the training objective of decoder $m$, such as cross-entropy for text and labels or a denoising objective for diffusion and flow decoders. With $\mathcal S=\mathcal M$, this trains each decoder on the full-observation latents required by Equation~\ref{eq:conditional-joint}, while smaller subsets further stabilize the learning of the shared representation.

To favor information useful across modalities, we use target-detached self-reconstruction~\citep{yeo2026muni}.
When the target is also an input ($m\in\mathcal S$), it still contributes to the latent, but the reconstruction gradients through its token stream are stopped.
Consequently, the target's own reconstruction loss cannot train its token stream to pass target-private information to the latent.
We describe the implementation in Section~\ref{sec:architecture}.

\begin{figure}[t]
\centering
\begin{tikzpicture}[
  x=1cm,y=1cm,>=Latex,
  font=\fontsize{8}{9.2}\selectfont,
  every node/.style={inner sep=1.5pt,outer sep=0pt},
  obs/.style={circle,draw=black!68,fill=white,minimum size=9mm,inner sep=.3pt},
  model/.style={rectangle,rounded corners=1pt,draw=latentblue,fill=latentfill,
    minimum width=9mm,minimum height=6.3mm,align=center,inner sep=3pt},
  decoder/.style={model},
  loss/.style={rectangle,rounded corners=1pt,draw=black!62,fill=black!3,
    minimum width=10mm,minimum height=6.3mm,inner sep=3pt},
  small/.style={font=\fontsize{8}{9.2}\selectfont},
  title/.style={font=\fontsize{9}{10.5}\selectfont\bfseries,anchor=west,inner sep=0pt},
  arrow/.style={-{Latex[length=1.6mm,width=1.15mm]},line width=.55pt},
  stop/.style={postaction={decorate},decoration={markings,
    mark=at position .53 with {
      \draw[-,line width=.65pt] (-2pt,-2.5pt)--(0pt,2.5pt);
      \draw[-,line width=.65pt] (0pt,-2.5pt)--(2pt,2.5pt);
    }}},
]
\draw[black!19,line width=.4pt] (4.22,.20)--(4.22,4.02);
\draw[black!19,line width=.4pt] (9.61,.20)--(9.61,4.02);
\node[title] at (.05,4.22) {(a) Reconstruction};
\node[title] at (4.45,4.22) {(b) Self-distillation};
\node[title] at (9.84,4.22) {(c) Contrastive alignment};

\node[obs] (other) at (1.20,3.50) {\scalebox{.90}{$x^{\mathcal S\!\setminus\!\{m\}}$}};
\node[obs] (target) at (3.45,3.50) {$x^m$};
\node[model] (recE) at (1.20,2.20) {$\mathcal E_\theta$};
\node[decoder] (recD) at (1.20,.60) {$\mathcal D_{\psi_m}$};
\node[loss] (recL) at (3.45,.60) {$\mathcal L_{\mathrm{rec}}$};
\node[small] (recQ) at (.23,2.20) {$\epsilon,0$};
\coordinate (targetbranch) at (3.45,2.77);
\coordinate (inputmerge) at (1.20,2.77);
\draw[line width=.55pt] (other.south)--(inputmerge);
\draw[arrow] (inputmerge)--(recE.north);
\draw[arrow] (recQ.east)--(recE.west);
\draw[arrow] (target.south)--(recL.north);
\fill (targetbranch) circle[radius=.85pt];
\draw[line width=.55pt,stop] (targetbranch)--(inputmerge);
\fill (inputmerge) circle[radius=.85pt];
\node[small] at (2.43,3.04) {$m\in\mathcal S$};
\draw[arrow] (recE.south)--(recD.north);
\draw[arrow] (recD.east)--(recL.west);

\begin{scope}[shift={(4.42,0)}]
\node[obs] (available) at (1.42,3.50) {$x^{\mathcal A}$};
\node[model,minimum width=10mm] (teacher) at (1.42,2.20)
  {{\fontsize{7}{8}\selectfont Teacher}\\$\mathcal E_\theta$};
\node[model,minimum width=10mm] (student) at (3.92,2.20)
  {{\fontsize{7}{8}\selectfont Student}\\$\mathcal E_\theta$};
\node[obs] (condition) at (3.92,3.50) {$x^{\mathcal S}$};
\node[loss] (distL) at (2.67,.60) {$\widetilde{\mathcal L}_{\mathrm{dist}}$};
\draw[arrow] (available.south)--(teacher.north);
\draw[arrow] (condition.south)--(student.north);
\draw[arrow,stop] (teacher.east)--(student.west);
\node[small] at (2.67,2.57) {$w_t^{\mathcal A}$};
\draw[arrow]
  (teacher.north west)
  .. controls (.45,2.86) and (.16,2.75) .. (.16,2.20)
  .. controls (.16,1.65) and (.45,1.54) .. (teacher.south west);
\node[small] at (.53,2.20) {$\times N$};
\node[small] (teacherPred) at (1.42,1.16)
  {$\mathcal E_\theta(w_t^{\mathcal A},t;x^{\mathcal A})$};
\node[small] (studentPred) at (3.92,1.16)
  {$\mathcal E_\theta(w_t^{\mathcal A},t;x^{\mathcal S})$};
\draw[arrow] (teacher.south)--(teacherPred.north);
\draw[arrow] (student.south)--(studentPred.north);
\draw[arrow,stop,rounded corners=2pt]
  (teacherPred.south)--(1.42,.60)--(distL.west);
\draw[arrow,rounded corners=2pt]
  (studentPred.south)--(3.92,.60)--(distL.east);
\end{scope}

\begin{scope}[shift={(9.85,0)}]
\foreach \k in {2,1}{
  \node[obs,draw=black!58,line width=.45pt]
    (viewSback\k) at ({.69+0.14*\k},{3.4+0.07*\k}) {};
  \node[obs,draw=black!58,line width=.45pt]
    (viewCback\k) at ({3.05+0.14*\k},{3.4+0.07*\k}) {};
  \node[model,draw=latentblue!85!white,line width=.45pt]
    (modelSback\k) at ({.69+0.14*\k},{2.12+0.07*\k}) {};
  \node[model,draw=latentblue!85!white,line width=.45pt]
    (modelCback\k) at ({3.05+0.14*\k},{2.12+0.07*\k}) {};
}
\node[obs] (viewS) at (.69,3.4) {$x_i^{\mathcal S}$};
\node[obs] (viewC) at (3.05,3.4) {$x_i^{\mathcal C}$};
\node[model] (conES) at (.69,2.12) {$\mathcal E_\theta$};
\node[model] (conEC) at (3.05,2.12) {$\mathcal E_\theta$};
\node[small,inner sep=.5pt] (queryS) at (1.42,2.84)
  {$\epsilon_i^{\mathcal S},0$};
\node[small,inner sep=.5pt] (queryC) at (3.78,2.84)
  {$\epsilon_i^{\mathcal C},0$};
\node[small] (embeddingS) at (.69,1.25) {$r_i^{\mathcal S}$};
\node[small] (embeddingC) at (3.05,1.25) {$r_i^{\mathcal C}$};
\node[loss] (conL) at (1.87,.60) {$\mathcal L_{\mathrm{con}}$};
\draw[arrow] (viewS.south)--([yshift=1.4mm]conES.north);
\draw[arrow] (viewC.south)--([yshift=1.4mm]conEC.north);
\draw[arrow] (conES.south)--(embeddingS.north);
\draw[arrow] (conEC.south)--(embeddingC.north);
\draw[arrow,rounded corners=2pt]
  (embeddingS.south)--(.69,.60)--(conL.west);
\draw[arrow,rounded corners=2pt]
  (embeddingC.south)--(3.05,.60)--(conL.east);
\end{scope}
\end{tikzpicture}
\vspace{-10pt}
\caption{\textbf{MUNITE training objectives.}
(a) Reconstruction learns a shared representation, using target detaching to discourage copying modality-private information.
(b) Self-distillation uses detached predictions conditioned on more modalities to supervise subset-conditioned predictions at the same noisy latent state, with both branches using the same model.
(c) An auxiliary contrastive loss encourages semantic alignment between complementary modality subsets using in-batch negatives.}
\label{fig:pipeline}
\end{figure}
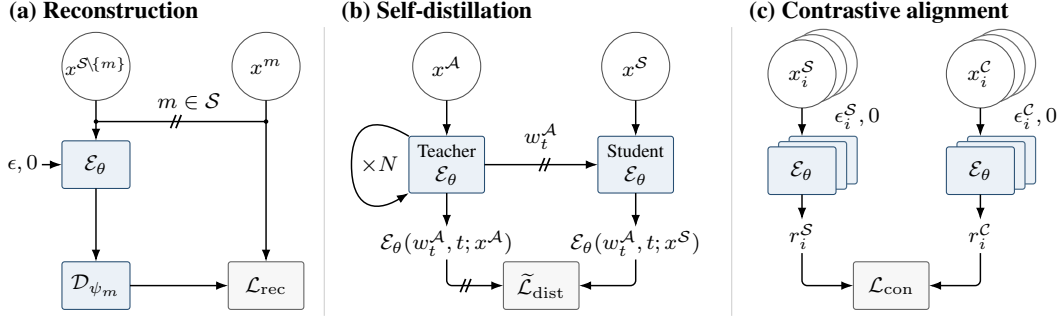

\headpre
\subsection{Conditional Flow Matching with Partial Observations}
\label{sec:distillation}
\headpost

We train $\mathcal E_\theta$ to model $Q_{\mathcal E_\theta}(\cdot\mid x^{\mathcal S})$ by conditional flow matching along a linear interpolant between standard Gaussian noise and
clean latents. Since the clean target can be computed only from a complete example, we extend the objective to partial observations through self-distillation.

\paragraph{Fully observed examples.}

The full-observation prediction provides the clean target $w=\mathcal E_\theta(\epsilon,0;x^{\mathcal M})$, where
$\epsilon\sim\mathcal N(0,I_D)$.
We use the same Gaussian sample as the initial flow state, setting $w_0=\epsilon$, and sample $t\sim\mathcal U[0,1)$.
The interpolant and the velocity field implied by the clean prediction are
\begin{equation}
w_t=(1-t)w_0+t\sg(w),\qquad
v_\theta(w_t,t;x^{\mathcal S})
=\frac{\mathcal E_\theta(w_t,t;x^{\mathcal S})-w_t}{1-t},
\label{eq:path}
\end{equation}
where $\sg$ denotes stop-gradient. Writing $W,W_0,W_t$ for the corresponding random variables, the conditional flow matching objective~\citep{lipman2023flowmatching} becomes
\begin{equation}
\mathcal L_{\mathrm{CFM}}
=\mathbb E\big[\|v_\theta(W_t,t;X^{\mathcal S})-(\sg(W)-W_0)\|_2^2\big]
=\mathbb E\left[\frac{\|\mathcal E_\theta(W_t,t;X^{\mathcal S})-\sg(W)\|_2^2}{(1-t)^2}\right].
\label{eq:cfm}
\end{equation}
Clean-latent prediction allows encoding and denoising to share their output space, as in JiT and UNITE~\citep{li2025jit,duggal2026unite}. For optimization stability, we drop the factor $(1-t)^{-2}$ and train with the unweighted denoising loss
\begin{equation*}
\widetilde{\mathcal L}_{\mathrm{CFM}}
=\mathbb E\big[\|\mathcal E_\theta(W_t,t;X^{\mathcal S})-\sg(W)\|_2^2\big].
\end{equation*}
At a fixed encoder and time $t$, both weightings have the population-optimal prediction~\citep{lipman2023flowmatching}
\begin{equation}
\mu_{\mathcal S}(w_t,t;x^{\mathcal S})
:=\mathbb E[W\mid W_t=w_t,\,X^{\mathcal S}=x^{\mathcal S}].
\label{eq:mmse}
\end{equation}
With full observations, this reduces to $\mathcal E(x^{\mathcal M})$, recovering encoding as a one-step prediction (Section~\ref{sec:shared-latent}). As in UNITE, the Gaussian query keeps the input format common to all observation levels; since the reconstruction target does not depend on it, the trained models effectively ignore it (Appendix~\ref{app:query-invariance}).

\paragraph{Partially observed examples.}
When $\mathcal A\subsetneq\mathcal M$, we cannot compute $\mathcal E(x^{\mathcal M})$. The available modalities nevertheless constrain its conditional distribution. For nested subsets $\mathcal S\subsetneq\mathcal A$, the tower property gives, at each fixed $t$,
\begin{equation}
\mathbb E\big[\mu_{\mathcal A}(W_t,t;X^{\mathcal A})\mid W_t,X^{\mathcal S}\big]
=\mu_{\mathcal S}(W_t,t;X^{\mathcal S}).
\label{eq:tower}
\end{equation}
A clean prediction using $x^{\mathcal A}$ can therefore supervise the prediction using $x^{\mathcal S}$, provided their common noisy state follows the full-latent intermediate distribution. We approximate it by integrating $\mathcal E_\theta$ conditioned on all available modalities up to the sampled time, using $N$ Euler steps with $N=4$ for training efficiency:
\begin{equation}
\frac{\mathrm{d}w_u^{\mathcal A}}{\mathrm{d}u}
=v_\theta(w_u^{\mathcal A},u;x^{\mathcal A}),\qquad
w_0^{\mathcal A}=w_0,\qquad 0\le u\le t.
\label{eq:teacher-path}
\end{equation}
At the resulting state, the same model serves as the teacher, conditioned on $x^{\mathcal A}$, and as the student, conditioned on $x^{\mathcal S}$, giving the unweighted distillation objective
\begin{equation}
\widetilde{\mathcal L}_{\mathrm{dist}}
=\mathbb E\left[
\left\|
\mathcal E_\theta(\sg(W_t^{\mathcal A}),t;X^{\mathcal S})
-\sg\left(
\mathcal E_\theta(W_t^{\mathcal A},t;X^{\mathcal A})
\right)
\right\|_2^2
\right].
\label{eq:distillation}
\end{equation}
For a complete example, the teacher target is its encoding $w$ and the state is the analytic interpolant in Equation~\ref{eq:path}, recovering $\widetilde{\mathcal L}_{\mathrm{CFM}}$.

\begin{proposition}[Gradient equivalence]\label{prop}
Let $\mathcal S\subsetneq\mathcal A$, let $x^{\mathcal A}$ be drawn from $P_{\mathrm{data}}(\mathrm{d}x^{\mathcal A})$, and treat the target latent distribution and teacher as fixed. If the teacher predicts $\mu_{\mathcal A}$ and its trajectory reproduces the conditional time marginals of $W_t$, 
then $\widetilde{\mathcal L}_{\mathrm{dist}}$ and $\widetilde{\mathcal L}_{\mathrm{CFM}}$ yield identical expected gradients through the subset-conditioned prediction.
\end{proposition}

The proof follows from the orthogonality of conditional expectation (Appendix~\ref{app:gradient-equivalence}). While the equivalence holds at a fixed representation with an exact teacher, in practice the current model serves as the teacher and its conditional flow is integrated numerically.

\headpre
\subsection{Contrastive Alignment across Observation Subsets}
\label{sec:contrastive}
\headpost

Reconstruction and denoising train $\mathcal E$ to retain cross-modal information and to model its conditional distributions, but neither directly shapes the latent geometry to reflect semantic similarity. We therefore add an auxiliary contrastive objective that promotes semantic alignment across observation subsets. For a batch of $B$ examples with available set $\mathcal A$, choose a nonempty proper subset $\mathcal S\subsetneq\mathcal A$ and its complement $\mathcal C=\mathcal A\setminus\mathcal S$. Let $r_i^{\mathcal S}$ and $r_i^{\mathcal C}$ be the flattened, $\ell_2$-normalized outputs of $\mathcal E_\theta(\epsilon_i^{\mathcal S},0;x_i^{\mathcal S})$ and $\mathcal E_\theta(\epsilon_i^{\mathcal C},0;x_i^{\mathcal C})$, using independent Gaussian queries. With temperature $\tau_c>0$, the directional InfoNCE loss~\citep{oord2018cpc} is
\begin{equation}
\mathcal L_{\mathcal S\to\mathcal C}
=-\frac1B\sum_{i=1}^B
\log\frac{\exp((r_i^{\mathcal S})^\top r_i^{\mathcal C}/\tau_c)}
{\sum_{j=1}^B\exp((r_i^{\mathcal S})^\top r_j^{\mathcal C}/\tau_c)}.
\label{eq:contrastive}
\end{equation}
We average the two directions, $\mathcal L_{\mathrm{con}}=(\mathcal L_{\mathcal S\to\mathcal C}+\mathcal L_{\mathcal C\to\mathcal S})/2$. As in Section~\ref{sec:reconstruction}, these one-step predictions from partial observations are not samples of the target latent distribution, but because they share the backbone parameters, aligning them stabilizes the internal representation of $\mathcal E$.

\begin{figure}[t]
\centering
\begin{tikzpicture}[
  x=1cm,y=1cm,
  font=\fontsize{8}{9.2}\selectfont,
  every node/.style={inner sep=1.5pt,outer sep=0pt},
  title/.style={anchor=west,font=\fontsize{9}{10.5}\selectfont\bfseries,inner sep=0pt},
  label/.style={font=\fontsize{8}{9.2}\selectfont,align=center},
  small/.style={font=\fontsize{8}{9.2}\selectfont,align=center},
  obs/.style={circle,draw=black!68,fill=white,minimum size=7.5mm,inner sep=.5pt},
  module/.style={rectangle,rounded corners=1pt,draw=latentblue,fill=latentfill,
    line width=.55pt,align=center,inner sep=3pt,minimum height=6.3mm},
  film/.style={module,minimum width=14.400mm,minimum height=4.4mm,inner sep=1.5pt},
  arrow/.style={-{Latex[length=1.6mm,width=1.15mm]},line width=.55pt},
  kv/.style={arrow,draw=latentblue},
  wire/.style={line width=.55pt},
  stop/.style={postaction={decorate},decoration={markings,
    mark=at position .42 with {
      \draw[-,draw=latentblue,line width=.65pt] (-2pt,-2.5pt)--(0pt,2.5pt);
      \draw[-,draw=latentblue,line width=.65pt] (0pt,-2.5pt)--(2pt,2.5pt);
    }}}
]
\node[title] (titleA) at (.02,6.10) {(a) Shared inference network};
\node[title] (titleB) at (4.35,6.10) {(b) Transformer block};
\draw[black!19,line width=.4pt] (4.00,.20)--(4.00,5.82);

\node[obs] (observation) at (.78,5.12) {$x^{\mathcal S}$};
\node[obs] (noise) at (2.700,5.12) {$w_t$};
\foreach \x in {-.30,0,.30} {
  \draw[draw=black!60,fill=white,line width=.45pt]
    ({.78+\x-.12},3.99) rectangle ++(.24,.25);
  \draw[draw=latentblue,fill=latentfill,line width=.45pt]
    ({2.700+\x-.12},3.99) rectangle ++(.24,.25);
  \draw[draw=latentblue,fill=latentfill,line width=.45pt]
    ({2.700+\x-.12},2.040) rectangle ++(.24,.25);
}
\draw[arrow] (observation.south)--(.78,4.26);
\draw[arrow] (noise.south)--(2.700,4.26);
\node[module,minimum width=28mm,minimum height=9mm] (transformer) at (1.80,3.120)
  {Transformer block $\times L$};
\draw[arrow] (.78,3.96)--(transformer.north -| .78,0);
\draw[arrow] (2.700,3.96)--(transformer.north -| 2.700,0);
\draw[arrow] (transformer.south -| 2.700,0)--(2.700,2.310);
\node[module,minimum width=22mm] (projection) at (2.700,1.370) {Output projection};
\draw[arrow] (2.700,2.010)--(projection.north);
\node[label] (prediction) at (2.700,.50) {$\mathcal E_\theta(w_t,t;x^{\mathcal S})$};
\draw[arrow] (projection.south)--(prediction.north);

\draw[draw=black!23,line width=.4pt,rounded corners=2pt]
  (4.43,.55) rectangle (13.72,5.78);
\node[label] (headN) at (5.65,5.48) {Modality $n$};
\node[label] (headR) at (8.96,5.48) {$K$ latent registers};
\node[label] (headM) at (12.27,5.48) {Modality $m$};
\foreach \cx/\tone in {5.65/white,8.96/latentfill,12.27/white} {
  \foreach \x in {-.30,0,.30} {
    \draw[draw=latentblue,fill=\tone,line width=.45pt]
      ({\cx+\x-.12},4.86) rectangle ++(.24,.25);
    \draw[draw=latentblue,fill=\tone,line width=.45pt]
      ({\cx+\x-.12},0.788) rectangle ++(.24,.25);
  }
}
\node[module,minimum width=21mm] (selfN) at (5.65,3.048) {Self-attention};
\node[module,minimum width=24mm] (latentAttention) at (8.96,3.048) {Latent attention};
\node[module,minimum width=21mm] (selfM) at (12.27,3.048) {Self-attention};
\node[module,minimum width=21mm,minimum height=5.5mm] (ffnN) at (5.65,1.548) {SwiGLU};
\node[module,minimum width=24mm,minimum height=5.5mm] (ffnR) at (8.96,1.548) {SwiGLU};
\node[module,minimum width=21mm,minimum height=5.5mm] (ffnM) at (12.27,1.548) {SwiGLU};
\node[film] (filmAttention) at (8.96,4.405) {FiLM};
\node[film] (filmFFN) at (8.96,2.278) {FiLM};
\node[small] (timeAttention) at (10.150,4.405) {$t$};
\node[small] (timeFFN) at (10.150,2.278) {$t$};
\draw[arrow] (timeAttention.west)--(filmAttention.east);
\draw[arrow] (timeFFN.west)--(filmFFN.east);
\node[small] (qkvN) at (5.65,4.440) {$Q,K,V$};
\node[small] (qkvR) at (8.96,3.774) {$Q,K,V$};
\node[small] (qkvM) at (12.27,4.440) {$Q,K,V$};
\draw[arrow] (5.65,4.86)--(qkvN.north);
\draw[arrow] (8.96,4.86)--(filmAttention.north);
\draw[arrow] (filmAttention.south)--(qkvR.north);
\draw[arrow] (12.27,4.86)--(qkvM.north);
\draw[arrow] (qkvN.south)--(selfN.north);
\draw[arrow] (qkvR.south)--(latentAttention.north);
\draw[arrow] (qkvM.south)--(selfM.north);
\coordinate (tapN) at (5.65,4.080);
\coordinate (tapM) at (12.27,4.080);
\fill (tapN) circle[radius=.75pt];
\fill (tapM) circle[radius=.75pt];
\draw[kv,rounded corners=2.5pt]
  (tapN)--(8.22,4.080)--([xshift=-7.4mm]latentAttention.north);
\draw[kv,stop,rounded corners=2.5pt]
  (tapM)--(9.70,4.080)--([xshift=7.4mm]latentAttention.north);
\node[small] (kvN) at (7.02,4.380) {$K_n,V_n$};
\node[small] (kvM) at (10.91,4.380) {$K_m,V_m$};
\draw[arrow] (selfN.south)--(ffnN.north);
\draw[arrow] (latentAttention.south)--(filmFFN.north);
\draw[arrow] (filmFFN.south)--(ffnR.north);
\draw[arrow] (selfM.south)--(ffnM.north);
\draw[arrow] (ffnN.south)--(5.65,1.058);
\draw[arrow] (ffnR.south)--(8.96,1.058);
\draw[arrow] (ffnM.south)--(12.27,1.058);

\draw[wire,draw=latentblue,stop] (4.64,.20)--(5.30,.20);
\node[small,anchor=west] (stopLegend) at (5.48,.20)
  {Stop target K/V gradients during reconstruction};
\path[use as bounding box] (0,0) rectangle (13.80,6.30);
\end{tikzpicture}
\caption{\textbf{Latent inference architecture.}
A shared Transformer performs latent inference under arbitrary observation subsets by aggregating modality information in latent registers. Keeping modality streams separate allows reconstruction gradients to be stopped on the target's key/value paths to the registers (slashes), while retaining its information in the forward pass.}
\label{fig:architecture}
\end{figure}
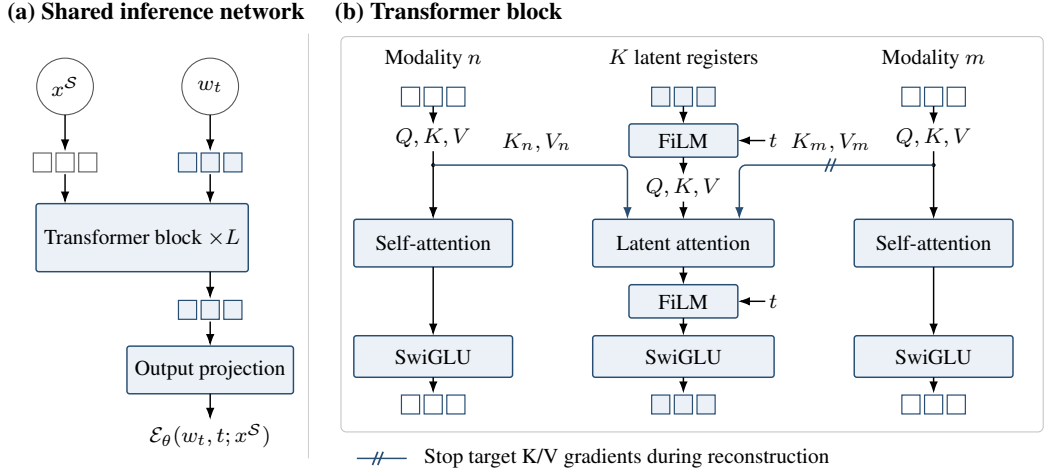

\headpre
\subsection{Model Architecture}
\label{sec:architecture}
\headpost

We implement $\mathcal E$ as a shared Transformer that predicts a clean latent from a latent state and an arbitrary subset of observed modalities.
Figure~\ref{fig:architecture} illustrates the architecture. Observed modalities are represented as separate token streams, while $K$ latent registers receive the Gaussian query or noisy latent state and aggregate information from the observed streams at every block. The final register states are projected to the clean latent prediction.

Within each block, modality streams attend only within themselves, whereas the latent registers attend to all observed streams and to one another, following the latent-attention construction of Perceiver~\citep{jaegle2021perceiver}. This separation prevents direct mixing between modality streams and allows their features to be reused across denoising steps.

For reconstruction target $m$, we detach its projected keys and values on every modality-to-register attention edge. The forward computation is unchanged, but reconstruction gradients can no longer flow from the registers into the target stream. Since modality streams do not attend to one another, no other modality offers an indirect path. This realizes the target-detached reconstruction of Section~\ref{sec:reconstruction} without target-specific encoders. Appendix~\ref{app:architecture} gives further implementation details.

\headpre
\subsection{Training and Inference}
\label{sec:training-generation}
\headpost

We train $\mathcal E$ and the modality-specific generative decoders with
\begin{equation}
\mathcal L
=\lambda_{\mathrm{rec}}\mathcal L_{\mathrm{rec}}
+\lambda_{\mathrm{dist}}\widetilde{\mathcal L}_{\mathrm{dist}}
+\lambda_{\mathrm{con}}\mathcal L_{\mathrm{con}}.
\label{eq:objective}
\end{equation}
The reconstruction loss updates $\mathcal E$ and the decoder of the sampled target, while the distillation and contrastive losses update only $\mathcal E$.
Algorithm~\ref{alg:training} in Appendix~\ref{app:subset-sampling} gives the full training procedure.

At inference, one latent drawn by integrating $v_\theta$ from Gaussian noise is shared by all target decoders, as the factorization in Equation~\ref{eq:conditional-joint} requires.

\headpre
\subsection{Relation to Multimodal Latent-Variable Models}
\label{sec:relation-to-multimodal}
\headpost

Among multimodal VAEs, which likewise decode modalities independently from a shared latent~\citep{wu2018mvae,shi2019mmvae}, JNF~\citep{senellart2025inference} is particularly close to our formulation: it first learns a shared latent with a joint VAE trained on complete observations, then fits a separate normalizing-flow posterior for each modality to samples from the joint encoder. Writing these posteriors as $q_m(z\mid x^m)$, inference from a subset $\mathcal S$ is based on their product,
\begin{equation}
q_{\mathrm{JNF}}(z\mid x^{\mathcal S})
\propto
p(z)^{1-|\mathcal S|}
\prod_{m\in\mathcal S}q_m(z\mid x^m),
\end{equation}
from which JNF samples using HMC.

MUNITE instead directly models $Q_{\mathcal E}(\mathrm{d}w\mid x^{\mathcal S})$, the conditional distribution of the same full-observation latent for every $\mathcal S$, without composing modality-specific posteriors.
Whereas JNF builds on a joint encoder trained on complete observations, MUNITE can also learn these conditionals from partially observed examples: the prediction under richer available evidence serves as a self-distillation target, justified by Equation~\ref{eq:tower} and Proposition~\ref{prop}.

MUNI~\citep{yeo2026muni} shares our use of factorized generative decoders, target-detached reconstruction, and a single latent sample for all targets, but draws this sample from an aggregated posterior $q_{\phi,\mathcal S}(z\mid x^{\mathcal S})$ for an observed subset. The central condition of its predictive-sufficiency criterion is
\begin{equation}
I(X^{\mathcal S};X^m\mid Z_{\mathcal S})=0
\qquad
\forall m\notin\mathcal S,
\end{equation}
which guarantees target-wise sufficiency but does not imply conditional independence among multiple missing targets. MUNITE instead samples $W\sim Q_{\mathcal E}(\cdot\mid x^{\mathcal S})$, where the full-observation latent $W$ is trained to satisfy Equation~\ref{eq:sufficiency}. The missing variation shared across targets is therefore sampled once in $W$, and marginalizing this latent yields the conditional joint in Equation~\ref{eq:conditional-joint}. We formalize this distinction in Appendix~\ref{app:conditional-coherence}.

\headpre
\section{Related Work}
\headpost

Since Section~\ref{sec:shared-latent} covers latent diffusion, Unified Latents, and UNITE~\citep{rombach2022ldm,heek2026ul,duggal2026unite}, we focus here on multimodal generation.

\paragraph{Any-to-any generation.}
Any-to-any generation extends cross-modal synthesis beyond fixed modality pairs to arbitrary combinations of observed and target modalities~\citep{tang2023codi,bachmann2024fourm21}. Masked models such as 4M and 4M-21 learn these mappings by predicting target-modality tokens from observed-modality tokens with a shared Transformer~\citep{mizrahi2023fourm,bachmann2024fourm21}. Joint flow models instead couple modality trajectories during sampling, using continuous representations in OmniFlow and discrete tokens in NExT-OMNI~\citep{li2025omniflow,luo2025nextomni}; our CFM and DFM baselines implement these two approaches in a controlled setting (Section~\ref{sec:experiments}).
CoDi similarly couples modality-specific denoisers through cross-attention to aligned representations of the other modalities' noisy latents~\citep{tang2023codi}.

\paragraph{Shared-latent generative models.}
Multimodal VAEs model multiple modalities through a shared latent variable and modality-specific encoders and decoders~\citep{wu2018mvae,shi2019mmvae,sutter2021mopoe,palumbo2023mmvaeplus,vo2026helvae}. Building on this factorization, MUNI jointly trains modality-specific encoders, expressive generative decoders, and a shared flow prior, and introduces coherence sufficiency, predictive sufficiency, and minimality as criteria for the learned multimodal latent~\citep{yeo2026muni}. FlowBind takes a different latent-space approach, using bidirectional flows to connect modality representations through a common latent distribution learned from partially paired data~\citep{cha2026flowbind}.
MUNITE instead directly models the conditional distribution of a coherence-sufficient full-observation latent under arbitrary observed subsets, replacing the diagonal-Gaussian parameterization of subset posteriors with conditional transport.

\headpre
\section{Experiments}
\label{sec:experiments}
\headpost
We evaluate one-to-one, many-to-one, one-to-many, and many-to-many generation, as well as unconditional co-generation.
Our central question is whether jointly generated targets agree on content that their inputs leave unspecified.
We adapt the PolyMNIST--digit--quadrant (PolyMNIST-D-Q) and image--text--audio benchmarks of MUNI~\citep{yeo2026muni} and introduce a new FFHQ64 benchmark with three image-like and two categorical modalities.

\headpre
\subsection{Experimental setup}
\headpost
We compare MUNITE with MUNI~\citep{yeo2026muni} and joint continuous and discrete flow-matching baselines, CFM and DFM~\citep{lipman2023flowmatching,gat2024dfm}. We train MUNI with its released training and sampling pipelines and retain its modality-specific decoder architectures in MUNITE.
To compare shared-latent and joint generation, CFM and DFM implement the continuous and discrete flow-matching approaches used in OmniFlow and NExT-OMNI, respectively~\citep{li2025omniflow,luo2025nextomni}. Both baselines use the same Transformer blocks as MUNITE but apply joint attention across all modality tokens. Since they do not use generative decoders, we increase the backbone depth to approximately match MUNITE's total number of trainable parameters. Because fully Transformer-based joint models may converge more slowly than latent models with generative decoders, we also train CFM and DFM for twice as many epochs as MUNITE (Appendix~\ref{app:baseline-controls}). CFM represents categorical modalities with one-hot vectors, whereas DFM discretizes continuous modalities with separately trained, frozen VQ-VAEs~\citep{oord2017vqvae}.
All four methods share the same training data, evaluation protocol, and frozen evaluators.

For the image--text--audio task, we additionally compare against FlowBind~\citep{cha2026flowbind} and pretrained CoDi and OmniFlow~\citep{tang2023codi,li2025omniflow}. Because CoDi and OmniFlow are pretrained on substantially larger datasets, we use their released models without retraining, take their single-target and unconditional scores from MUNI~\citep{yeo2026muni}, and evaluate their one-to-many coherence ourselves (Appendix~\ref{app:external-baselines}). The released FlowBind checkpoint is trained on less data than our benchmark, so we retrain it on our training data with its official pipeline.

The training data mix fully and partially observed examples (Table~\ref{tab:training-data} in Appendix~\ref{app:protocol}). For PolyMNIST-D-Q and FFHQ64, we mask one or two of the three non-categorical modalities in three quarters of the training examples and keep the categorical labels. For image--text--audio, the corpora used by MUNI and FlowBind contain only modality pairs, so we add caption--image--audio triplets from E-MM1~\citep{broadbent2025ebind}. For evaluation, we use 10,000 complete PolyMNIST tuples, 5,000 complete FFHQ tuples, and the image--text--audio evaluation sets of MUNI~\citep{yeo2026muni}.
Appendix~\ref{app:qualitative} shows qualitative results.

\headpre
\subsection{PolyMNIST-D-Q}
\headpost
\begin{table}[t]
\caption{Results on PolyMNIST-D-Q (top) and FFHQ64 (bottom). Acc. and Coh. denote accuracy and coherence, $d$ and $q$ denote digit and quadrant, and scores for $m_i$ are averaged over the three views. FD uses verifier features, and Normal Err. is $1-\cos$. Best results are in bold.}
\label{tab:image-main}\centering\small
\setlength{\tabcolsep}{3pt}
\begin{tabularx}{\linewidth}{@{}llYrrrr@{}}\toprule
Route & Input $\to$ output & Metric & MUNI & CFM & DFM & MUNITE\\\midrule
\multicolumn{7}{c}{\textit{PolyMNIST-D-Q}}\\
\multirow{3}{*}{Uncond.} & \multirow{3}{*}{All five modalities} & FD $\downarrow$ & 3.0407 & 71.6626 & 19.9587 & \textbf{1.5230}\\
& & Digit Coh. (all) $\uparrow$ & 0.5799 & 0.0237 & 0.2973 & \textbf{0.9826}\\
& & Quad. Coh. (all) $\uparrow$ & 0.9041 & 0.9902 & 0.8222 & \textbf{1.0000}\\\midrule
\multirow{2}{*}{$1\to1$}
& $d\to m_i$ & Digit Acc. $\uparrow$ & 0.9843 & 0.2402 & 0.7224 & \textbf{0.9941}\\
& $q\to m_i$ & Quad. Acc. $\uparrow$ & \textbf{1.0000} & 0.9936 & 0.9372 & \textbf{1.0000}\\\midrule
$N\to1$ & $(d,q)\to m_i$ & Both-label Acc. $\uparrow$ & 0.9925 & 0.2200 & 0.7325 & \textbf{0.9943}\\\midrule
\multirow{2}{*}{$1\to N$} & $d\to(m_0,m_1,m_2)$ & Quad. Coh. $\uparrow$ & 0.2528 & 0.9879 & 0.8798 & \textbf{1.0000}\\
& $q\to(m_0,m_1,m_2)$ & Digit Coh. $\uparrow$ & 0.1074 & 0.1597 & 0.4612 & \textbf{0.9904}\\\midrule
\multicolumn{7}{c}{\textit{FFHQ64}}\\
\multirow{5}{*}{Uncond.} & \multirow{5}{*}{All five modalities} & RGB FID $\downarrow$ & 34.618 & 122.617 & 55.840 & \textbf{29.918}\\
& & Age Coh. $\uparrow$ & 0.5630 & 0.3068 & 0.4546 & \textbf{0.5692}\\
& & Gender Coh. $\uparrow$ & 0.9596 & 0.7328 & 0.8658 & \textbf{0.9610}\\
& & RGB--Seg. Coh. $\uparrow$ & 0.7762 & 0.7357 & 0.8084 & \textbf{0.8896}\\
& & RGB--Normal Err. $\downarrow$ & 0.0700 & 0.0603 & 0.0495 & \textbf{0.0491}\\\midrule
\multirow{2}{*}{$1\to1$} & Age $\to$ RGB & Age Acc. $\uparrow$ & 0.5204 & 0.2506 & 0.4002 & \textbf{0.6274}\\
& Gender $\to$ RGB & Gender Acc. $\uparrow$ & 0.9684 & 0.6944 & 0.8518 & \textbf{0.9790}\\\midrule
$N\to1$ & (Age, Gender) $\to$ RGB & Both Acc. $\uparrow$ & 0.5100 & 0.1842 & 0.3690 & \textbf{0.5898}\\\midrule
\multirow{2}{*}{$M\to N$} & \multirow{2}{*}{\shortstack[l]{(Age, Gender) $\to$\\(RGB, Seg., Normals)}} & RGB--Seg. Coh. $\uparrow$ & 0.6038 & 0.7441 & 0.8142 & \textbf{0.8888}\\
& & RGB--Normal Err. $\downarrow$ & 0.1106 & 0.0572 & \textbf{0.0492} & 0.0500\\\bottomrule
\end{tabularx}
\end{table}

The three $64\times64$ RGB views share digit identity and quadrant but differ in handwriting and background. Conditioning on either label leaves the other attribute unobserved, so joint generation requires the three views to agree on that missing attribute. For unconditional co-generation, coherence measures agreement among the three images and the jointly generated labels. Following MUNI~\citep{yeo2026muni}, we compute digit and quadrant accuracy and coherence using a frozen SegFormer verifier (Appendix~\ref{app:evaluators}), while Fr\'echet distance (FD) in verifier-feature space measures image fidelity.

MUNITE retains strong single-image label accuracy while substantially improving agreement on the unobserved attribute (Table~\ref{tab:image-main}, top). Under quadrant conditioning, MUNI and MUNITE both reach 1.0000 quadrant accuracy for single images, yet the digit coherence of their jointly generated views is 0.1074 and 0.9904, respectively. Under digit conditioning, quadrant coherence likewise rises from 0.2528 to 1.0000. MUNI's scores are near the chance agreement of independent views, 0.10 and 0.25, so accurate single-view conditionals alone do not make the views agree (Section~\ref{sec:relation-to-multimodal}). Each MUNITE view still follows the data distribution of the missing attribute, and the views agree only by chance when decoded from separate latent samples (Appendix~\ref{app:poly-results}).

MUNITE also achieves the lowest FD and the highest unconditional coherence for both labels. CFM and DFM attain high quadrant agreement on some routes, yet lag behind both latent models in digit accuracy and image fidelity despite matched trainable parameters and twice as many training epochs.
Their joint backbones must model the handwriting and background of each view together with the shared attributes, which both latent models leave to image-specific decoders.

\headpre
\subsection{FFHQ64}
\headpost
We construct a five-modality task from FFHQ~\citep{karras2019stylegan}, combining $64\times64$ RGB faces with the age, gender, and segmentation annotations of FFHQ-Aging~\citep{orel2020lifespan} and pseudo-labeled surface normals.
We evaluate RGB generation from either or both categorical attributes, joint generation of RGB, segmentation, and normals from both, and unconditional co-generation of all five modalities.
Attribute accuracy measures agreement with the observed labels, and coherence measures agreement among jointly generated modalities. RGB FID measures image fidelity, and normal error the geometric consistency between generated normals and RGB.

On FFHQ64, the gains extend from shared labels to dense cross-modal structure (Table~\ref{tab:image-main}, bottom). When conditioning on age and gender, MUNITE increases RGB--segmentation coherence from 0.6038 to 0.8888 and reduces normal error from 0.1106 to 0.0500 relative to MUNI. DFM attains a slightly lower normal error (0.0492 vs.~0.0500), but MUNITE attains much higher segmentation coherence (0.8888 vs.~0.8142).
MUNITE also has the lowest unconditional RGB FID and the strongest unconditional coherence on all four cross-modal comparisons. CFM and DFM remain substantially weaker in RGB fidelity, as in PolyMNIST-D-Q.

\headpre
\subsection{Image--text--audio}
\label{sec:exp-media}
\headpost
We follow the image--text--audio benchmark used by FlowBind and MUNI~\citep{cha2026flowbind,yeo2026muni}, modeling fixed 768-dimensional features and evaluating the rendered images, captions, and audio.
For one-to-one generation, we use FID, FAD, and CIDEr for fidelity and CLIP, CLAP, and Audio-Image Similarity (AIS) for cross-modal alignment.
For many-to-one generation, we measure the alignment of the generated target with each source modality, and for one-to-many and unconditional co-generation, the alignment among jointly generated modalities, which we report as coherence.
\begin{table}[htbp]
\caption{Image--text--audio coherence between jointly generated modalities. $T$, $I$, and $A$ denote text, image, and audio. OmniFlow's unconditional scores are from \citet[Table~4]{yeo2026muni}, and dashes denote unsupported unconditional generation. Higher is better, and best results are in bold.}
\label{tab:media-coherence}\centering\small
\setlength{\tabcolsep}{4pt}
\begin{tabularx}{\linewidth}{@{}l*{6}{>{\centering\arraybackslash}X}@{}}\toprule
& \multicolumn{3}{c}{$1\to N$ coherence} & \multicolumn{3}{c}{Unconditional coherence}\\\cmidrule(lr){2-4}\cmidrule(l){5-7}
Method & $T\to(I,A)$ & $I\to(T,A)$ & $A\to(T,I)$ & CLIP & CLAP & AIS\\
& AIS & CLAP & CLIP & $(T,I)$ & $(T,A)$ & $(I,A)$\\\midrule
CoDi & 63.866 & 7.418 & 23.414 & \textemdash & \textemdash & \textemdash\\
OmniFlow & 77.027 & 14.781 & 22.697 & 21.17 & 14.23 & 50.95\\
\midrule
MUNI & 81.558 & 32.237 & 25.850 & 26.949 & 28.467 & 81.989\\
CFM & 74.702 & 21.695 & 24.034 & 23.622 & 19.896 & 74.208\\
DFM & 73.635 & 20.951 & 24.037 & 23.740 & 19.503 & 71.942\\
FlowBind & 79.819 & 25.812 & 24.853 & \textemdash & \textemdash & \textemdash\\
MUNITE & \textbf{83.336} & \textbf{34.762} & \textbf{26.334} & \textbf{27.398} & \textbf{30.220} & \textbf{87.045}\\
\bottomrule\end{tabularx}
\end{table}

MUNITE achieves the highest coherence in all three one-to-many directions and all three unconditional comparisons (Table~\ref{tab:media-coherence}). The improvement holds across text, image, and audio, indicating that the shared latent preserves dependencies among generated targets when their shared variation is not determined by the input. Ablating target detaching, contrastive alignment, or partial-data distillation leaves this one-to-many advantage over MUNI intact, whereas unconditional coherence relies most on target detaching and partial-data distillation (Appendix~\ref{app:media-ablations}). On the one-to-one and many-to-one routes, MUNITE is competitive with the compared baselines, including the externally pretrained models (Appendix~\ref{app:media-single-results}).

\headpre
\section{Conclusion and Future Work}
\headpost

MUNITE unifies encoding and any-to-any multimodal generation as inference over a full-observation latent. Self-distillation extends the learning of its conditional distributions to incomplete training examples, and modality-specific decoders generate from a shared latent sample. Across the evaluated benchmarks, MUNITE improves coherence among jointly generated modalities while maintaining competitive single-target generation quality.

A natural next step is multimodal world modeling.
The compact latent learned by MUNITE could serve as a shared world-state representation that can be inferred from arbitrary subsets of observations. Learning dynamics directly in this latent space would enable state estimation and future prediction under varying multimodal observations, while modality-specific decoders could render the predicted state into any desired modality.

\label{main-end}
\clearpage
\section*{Ethics statement}
\vspace{-0.5\baselineskip}
Our experiments include face and audiovisual generation, which may reproduce biases present in the underlying datasets and can be misused to create deceptive synthetic media. In the FFHQ64 benchmark, age and gender are used only as dataset annotations for controlled evaluation and should not be interpreted as validated personal attributes. Generated examples should therefore be interpreted as synthetic outputs rather than representations of real individuals or verified attributes.
\section*{Reproducibility statement}
\vspace{-0.5\baselineskip}
Section~\ref{sec:method} specifies the learning objectives, gradient routing, and inference procedure, and Appendix~\ref{app:gradient-equivalence} proves Proposition~\ref{prop}. Appendix~\ref{app:implementation} describes the architectures, the sampling of training subsets, and the training update in Algorithm~\ref{alg:training}. Appendix~\ref{app:protocol} details how each benchmark and its observation groups are constructed, together with model sizes, optimization and inference budgets, and the evaluators. Appendix~\ref{app:metrics} defines all metrics and gives the sources of the external baseline results.

\bibliography{iclr2027_conference}
\bibliographystyle{iclr2027_conference}
\clearpage
\appendix
\section*{Appendix}
\raggedbottom
\section{Full-observation latents and conditional generation}
\label{app:theory}
This section expands on the formulation of Section~\ref{sec:shared-latent}. We first characterize the latent conditionals induced by the full-observation encoder, then derive the conditional joint in Equation~\ref{eq:conditional-joint} and contrast it with the target-wise criterion of MUNI.

\subsection{The distribution induced by full observations}
For a fixed representation, $W=\mathcal E(X^{\mathcal M})$, so the conditional distributions in Section~\ref{sec:shared-latent} are pushforwards of the data conditionals through the full-observation encoder:
\begin{equation}
Q_{\mathcal E}(\cdot\mid x^{\mathcal S})
=\mathcal E_{\#}\big[P_{\mathrm{data}}(\cdot\mid x^{\mathcal S})\big].
\label{eq:app-pushforward}
\end{equation}
With full observation this is $\delta_{\mathcal E(x^{\mathcal M})}$, and empty conditioning gives the marginal distribution of $W$.

\paragraph{Query dependence.}\label{app:query-invariance}
As described in Section~\ref{sec:distillation}, the shared model also receives a Gaussian query under full observation, so the full-observation latent $\mathcal E_\theta(\epsilon,0;x^{\mathcal M})$ need not be exactly deterministic. We therefore measure its dependence on the query by the variance ratio
\begin{equation*}
r=\frac{\mathbb E_{x^{\mathcal M}}\operatorname{tr}\operatorname{Cov}_{\epsilon}\big[\mathcal E_\theta(\epsilon,0;x^{\mathcal M})\big]}
{\operatorname{tr}\operatorname{Cov}_{x^{\mathcal M}}\big[\mathbb E_{\epsilon}\,\mathcal E_\theta(\epsilon,0;x^{\mathcal M})\big]},
\end{equation*}
estimated from eight queries for each of 1,000 held-out complete observations (the 975 evaluation triplets for image--text--audio). The trained models give $r=2.6\times10^{-5}$ on PolyMNIST-D-Q, $7.9\times10^{-5}$ on FFHQ64, and $6.5\times10^{-4}$ on image--text--audio, so the full-observation latent is deterministic for practical purposes.

\subsection{Derivation of the conditional joint}
The coherence-sufficiency condition in Equation~\ref{eq:sufficiency} makes all modalities conditionally independent given $W$. Hence, for disjoint input and target sets $\mathcal S$ and $\mathcal T$,
\begin{equation*}
P_{\mathrm{data}}(\mathrm{d}x^{\mathcal T}\mid W=w,x^{\mathcal S})
=\prod_{m\in\mathcal T}P_{\mathrm{data}}(\mathrm{d}x^m\mid W=w),
\end{equation*}
and integrating over the latent conditional in Equation~\ref{eq:app-pushforward} gives
\begin{equation}
\begin{aligned}
P_{\mathrm{data}}(\mathrm{d}x^{\mathcal T}\mid x^{\mathcal S})
&=\int P_{\mathrm{data}}(\mathrm{d}x^{\mathcal T}\mid W=w,x^{\mathcal S})\,Q_{\mathcal E}(\mathrm{d}w\mid x^{\mathcal S})\\
&=\int\prod_{m\in\mathcal T}P_{\mathrm{data}}(\mathrm{d}x^m\mid W=w)\,Q_{\mathcal E}(\mathrm{d}w\mid x^{\mathcal S}),
\end{aligned}
\label{eq:app-conditional-joint}
\end{equation}
which is Equation~\ref{eq:conditional-joint}. One sample of $w$ must therefore be shared by all target decoders. Sampling a separate latent for each target would instead produce $\prod_{m\in\mathcal T}P_{\mathrm{data}}(\mathrm{d}x^m\mid x^{\mathcal S})$, discarding the dependence that remains after observing the inputs. Taking $\mathcal S=\emptyset$ covers unconditional co-generation.

\subsection{Predictive sufficiency and joint coherence}
\label{app:conditional-coherence}
MUNI trains the latent of an observed subset toward predictive sufficiency~\citep{yeo2026muni},
\begin{equation}
\mathrm{TC}(X^{\mathcal S}\mid Z_{\mathcal S})=0,
\qquad I(X^{\mathcal S};X^m\mid Z_{\mathcal S})=0\quad\text{for all }m\notin\mathcal S.
\label{eq:app-muni-criterion}
\end{equation}
The first condition, on the total correlation (TC), concerns dependence within the observed subset, and the second preserves its information about each missing modality. Neither requires the missing targets to be independent given $Z_{\mathcal S}$, as factorized decoding would. A latent can therefore satisfy Equation~\ref{eq:app-muni-criterion} while factorized decoding misses the joint conditional.

Consider independent uniform bits $U$ and $B$ with $X^1=U$ and $X^2=X^3=(U,B)$, and let $\mathcal S=\{1\}$. The latent $Z_{\mathcal S}=U$ satisfies Equation~\ref{eq:app-muni-criterion}, and pairing $U$ with an independent uniform bit reproduces each conditional $P_{\mathrm{data}}(\mathrm{d}x^m\mid x^1)$ exactly. When the two targets are decoded independently in this way, however, they coincide only with probability $1/2$, whereas $X^2$ and $X^3$ always coincide in the data. This holds even though $I(X^1;(X^2,X^3)\mid Z_{\mathcal S})=0$, so preserving the source information about the joint target does not justify independent decoding; the latent must also resolve the dependence between the targets. The full-observation latent $W=(U,B)$ satisfies Equation~\ref{eq:sufficiency}, and its conditional distribution
\begin{equation*}
Q_{\mathcal E}(\cdot\mid x^1=u)=\tfrac12\delta_{(u,0)}+\tfrac12\delta_{(u,1)}
\end{equation*}
fixes the shared bit once, before either target is decoded, so Equation~\ref{eq:app-conditional-joint} recovers the correct joint conditional. MUNI imposes coherence sufficiency on the encodings used for prior learning, whereas MUNITE requires it of the full-observation latent whose conditional distribution is inferred for every input subset.

\section{Derivation of partial-data self-distillation}
\label{app:gradient-equivalence}
We prove Proposition~\ref{prop} for the unweighted objectives used in training. Throughout, the representation and the teacher are fixed, both objectives use the same distribution over times and conditioning subsets, and each partially observed example is a complete draw from $P_{\mathrm{data}}$ with its unavailable modalities removed independently of the data. The synthetic masking of PolyMNIST-D-Q and FFHQ64 matches this setting, whereas on image--text--audio the available modalities follow the source corpus.

\subsection{The teacher state}
Fix $t\in[0,1)$. Let $W_0\sim\mathcal N(0,I_D)$ be independent of the complete observation and $W_t=(1-t)W_0+tW$. Since $W-W_t=(1-t)(W-W_0)$, the clean-latent prediction $\mu_{\mathcal A}$ determines the velocity
\begin{equation*}
\frac{\mu_{\mathcal A}(w_t,t;x^{\mathcal A})-w_t}{1-t}
=\mathbb E[W-W_0\mid W_t=w_t,X^{\mathcal A}=x^{\mathcal A}],
\end{equation*}
which is the marginal velocity of the interpolation given $X^{\mathcal A}$. By Theorem~1 of \citet{lipman2023flowmatching}, the ODE driven by this velocity from $W_0$ has the conditional time marginals of $W_t$, so an exact teacher gives
\begin{equation}
(W_t^{\mathcal A},X^{\mathcal A})\overset{d}{=}(W_t,X^{\mathcal A}).
\label{eq:app-teacher-marginals}
\end{equation}
A partial example thus supplies a correctly distributed evaluation state without revealing its full latent. Only these marginals must match; individual teacher trajectories may differ from the straight interpolation paths.

\subsection{Proof of Proposition~\ref{prop}}
\begin{proof}
Fix $t$ and write $f_\theta=\mathcal E_\theta(W_t,t;X^{\mathcal S})$ and $\mu_{\mathcal A}=\mu_{\mathcal A}(W_t,t;X^{\mathcal A})$. Since $\mathcal S\subsetneq\mathcal A$, both are functions of $(W_t,X^{\mathcal A})$, and $\mathbb E[W-\mu_{\mathcal A}\mid W_t,X^{\mathcal A}]=0$. The cross term in
\begin{equation*}
\mathbb E\|f_\theta-W\|_2^2
=\mathbb E\|f_\theta-\mu_{\mathcal A}\|_2^2+\mathbb E\|W-\mu_{\mathcal A}\|_2^2
+2\,\mathbb E\langle f_\theta-\mu_{\mathcal A},\mu_{\mathcal A}-W\rangle
\end{equation*}
therefore vanishes. By Equation~\ref{eq:app-teacher-marginals}, the first term on the right is the distillation loss at time $t$ with teacher prediction $\mu_{\mathcal A}$. Averaging over $t$ and the subsets gives
\begin{equation}
\widetilde{\mathcal L}_{\mathrm{CFM}}-\widetilde{\mathcal L}_{\mathrm{dist}}
=\mathbb E\|W-\mu_{\mathcal A}(W_t,t;X^{\mathcal A})\|_2^2,
\label{eq:app-unweighted-gap}
\end{equation}
which does not depend on $\theta$, so the two objectives have the same gradient with respect to $\theta$. By the tower property, their common minimizer over functions of $(W_t,X^{\mathcal S})$ is $\mathbb E[\mu_{\mathcal A}\mid W_t,X^{\mathcal S}]=\mathbb E[W\mid W_t,X^{\mathcal S}]=\mu_{\mathcal S}(W_t,t;X^{\mathcal S})$.
\end{proof}

\subsection{Full observations and subset averaging}
When $\mathcal A=\mathcal M$, the teacher target is $\mathcal E(x^{\mathcal M})$, and the state is the analytic interpolant in Equation~\ref{eq:path}. Under deterministic encoding, the residual in Equation~\ref{eq:app-unweighted-gap} vanishes, and the target is independent of the Gaussian query, so that query can also serve as $W_0$.

The implementation averages the loss over one subset of each size (Algorithm~\ref{alg:training}) instead of drawing a single size per update. By linearity of expectation, this stratified average equals the objective under uniformly sampled sizes, so the derivation applies to it unchanged.

\FloatBarrier
\section{Architecture and training details}
\label{app:implementation}

\subsection{Architecture}
\label{app:architecture}
Images are divided into $8\times8$ patches as in ViT~\citep{dosovitskiy2021vit}, giving 64 tokens per $64\times64$ image, and each categorical modality uses a learnable embedding table. For image--text--audio, every modality is given as a precomputed 768-dimensional feature vector, following MUNI and FlowBind, and a linear layer maps it to eight tokens.

\noindent\begin{minipage}[t]{0.49\linewidth}
\vspace{0pt}
\setlength{\parskip}{.5pc}
The final register states predict $K$ latent tokens of dimension $d$, with total dimension $D=Kd$. We use $K\times d=4\times8$ for PolyMNIST-D-Q, $64\times8$ for FFHQ64, and $1\times768$ for image--text--audio.

Table~\ref{tab:app-joint-config} lists the widths and depths of the Transformer backbones, which use eight attention heads and FiLM timestep conditioning~\citep{perez2018film} where applicable. Table~\ref{tab:app-capacity} compares the total trainable parameters, including the generative decoders.
\end{minipage}\hfill
\begin{minipage}[t]{0.49\linewidth}
\vspace{0pt}
\setlength{\abovecaptionskip}{0pt}
\setlength{\belowcaptionskip}{3pt}
\captionof{table}{Transformer backbones.}
\label{tab:app-joint-config}\centering\small
\setlength{\tabcolsep}{3pt}
\begin{tabularx}{\linewidth}{@{}Yrrc@{}}\toprule
Model & Width & Depth & Time cond.\\\midrule
\multicolumn{4}{@{}l}{\textit{PolyMNIST-D-Q}}\\
MUNITE & 256 & 7 & FiLM\\
CFM & 256 & 16 & FiLM\\
DFM & 256 & 24 & N/A\\\midrule
\multicolumn{4}{@{}l}{\textit{FFHQ64}}\\
MUNITE & 512 & 11 & FiLM\\
CFM & 512 & 22 & FiLM\\
DFM & 512 & 33 & N/A\\\midrule
\multicolumn{4}{@{}l}{\textit{Image--text--audio}}\\
MUNITE & 768 & 6 & FiLM\\
CFM & 768 & 26 & FiLM\\
DFM & 768 & 37 & N/A\\\bottomrule
\end{tabularx}

\end{minipage}\par

MUNI and MUNITE use the same generative-decoder families~\citep{yeo2026muni}. On PolyMNIST-D-Q, the three images share an NCSN++ decoder~\citep{song2021sde} of base width 64, and the labels use two-layer heads of width 256. On FFHQ64, the dense modalities use a spatial NCSN++ decoder of base width 128, and the labels use two-layer heads of width 512. Following common practice for image diffusion models~\citep{esser2024sd3}, both methods train these image decoders with logit-normal timestep sampling. For image--text--audio, we adopt MUNI's decoders and their training details without modification. Image and audio features are decoded by a width-768, depth-10 flow and a width-768, depth-6 feedforward network, respectively. Text is decoded by the Gemma-3-1B~\citep{gemmateam2025gemma3} feature-to-text decoder fine-tuned by FlowBind~\citep{cha2026flowbind}. A feedforward network maps the shared latent to its prefix, and the decoder itself is adapted only through LoRA~\citep{hu2022lora}.

\subsection{Training procedure}
\label{app:subset-sampling}
Each minibatch contains examples with the same set of available modalities $\mathcal A$, from which all conditioning subsets are drawn. Let $n=|\mathcal A|$. For reconstruction, a target $m$ is sampled uniformly from $\mathcal A$, and a conditioning subset $\mathcal S\subseteq\mathcal A$ is drawn by first sampling its size uniformly from $\{1,\ldots,n\}$. On image--text--audio, this rule is mixed equally with one that draws $\mathcal S$ from $\mathcal A\setminus\{m\}$, with size uniform on $\{0,\ldots,n-1\}$, to increase the share of cross-modal reconstruction. Contrastive splits draw the size of $\mathcal S$ uniformly from $\{1,\ldots,n-1\}$. Algorithm~\ref{alg:training} summarizes one update. We set $\lambda_{\mathrm{rec}}=1$, $\lambda_{\mathrm{dist}}=2$, $\lambda_{\mathrm{con}}=0.1$, and $\tau_c=0.1$.
\begin{algorithm}[t]
\caption{MUNITE training step}
\label{alg:training}\small
\begin{algorithmic}[1]
\Require minibatch $x^{\mathcal A}$, teacher Euler steps $N$, learning rate $\eta$
\State Sample $m\in\mathcal A$, $\mathcal S\subseteq\mathcal A$, and $\epsilon\sim\mathcal N(0,I_D)$
\State $\mathcal L_{\mathrm{rec}}\gets\ell_m\big(\mathcal D_{\psi_m}(\mathcal E_\theta(\epsilon,0;x^{\mathcal S})),x^m\big)$ \Comment{target-detached}
\State Sample $w_0\sim\mathcal N(0,I_D)$ and $t\sim\mathcal U[0,1)$
\If{$\mathcal A=\mathcal M$}
  \State $\bar w\gets\mathcal E_\theta(w_0,0;x^{\mathcal M})$;\ \ $w_t\gets(1-t)w_0+t\bar w$;\ \ $\mathcal K\gets\{0,\ldots,M\}$
\Else
  \State $w_t\gets w_0$
  \For{$j=0,\ldots,N-1$}
    \State $w_t\gets w_t+\frac tN\,v_\theta\big(w_t,\frac{jt}{N};x^{\mathcal A}\big)$
  \EndFor
  \State $\bar w\gets\mathcal E_\theta(w_t,t;x^{\mathcal A})$;\ \ $\mathcal K\gets\{0,\ldots,|\mathcal A|-1\}$
\EndIf
\State $\widetilde{\mathcal L}_{\mathrm{dist}}\gets0$
\For{$k\in\mathcal K$}
  \State Sample $\mathcal S_k\subseteq\mathcal A$ with $|\mathcal S_k|=k$
  \State $\widetilde{\mathcal L}_{\mathrm{dist}}\gets\widetilde{\mathcal L}_{\mathrm{dist}}+\frac{1}{|\mathcal K|}\big\|\mathcal E_\theta(\sg(w_t),t;x^{\mathcal S_k})-\sg(\bar w)\big\|_2^2$
\EndFor
\State $\mathcal L_{\mathrm{con}}\gets0$
\If{$|\mathcal A|>1$}
  \State Sample nonempty $\mathcal S\subsetneq\mathcal A$;\ \ $\mathcal C\gets\mathcal A\setminus\mathcal S$
  \State $\mathcal L_{\mathrm{con}}\gets\frac12(\mathcal L_{\mathcal S\to\mathcal C}+\mathcal L_{\mathcal C\to\mathcal S})$
\EndIf
\State $\mathcal L\gets\lambda_{\mathrm{rec}}\mathcal L_{\mathrm{rec}}+\lambda_{\mathrm{dist}}\widetilde{\mathcal L}_{\mathrm{dist}}+\lambda_{\mathrm{con}}\mathcal L_{\mathrm{con}}$
\State $(\theta,\psi)\gets(\theta,\psi)-\eta\,\nabla_{\theta,\psi}\mathcal L$
\end{algorithmic}
\end{algorithm}

\subsection{Baselines}
\label{app:baselines}
MUNI aggregates Gaussian modality encoders by a product of experts without a prior expert and samples this posterior directly for conditional generation~\citep{yeo2026muni}.

CFM and DFM use the Transformer blocks, width, and number of attention heads of MUNITE (Table~\ref{tab:app-joint-config}), with more blocks so that their trainable parameters approximately match those of MUNITE, including its generative decoders (Table~\ref{tab:app-capacity}). DFM has no timestep conditioning and therefore needs more blocks than CFM. CFM represents images and features as continuous patches or vectors and labels as one-hot vectors, and learns a straight flow from Gaussian noise~\citep{lipman2023flowmatching}. DFM represents each label by a single token and each image or feature by frozen VQ-VAE codes, with 1,024-entry codebooks, 64 codes per image, and 32 codes per image--text--audio feature, and is trained with absorbing-mask corruption and clean-token prediction~\citep{gat2024dfm}. During training, the numbers of source and target modalities, $s\ge0$ and $r\ge1$ with $s+r\le n$, are sampled uniformly, and the available modalities are randomly assigned to the source, target, and unused sets.

\FloatBarrier
\section{Benchmark construction and comparison protocol}
\label{app:protocol}

\subsection{Datasets}
\begin{table}[htbp]
\caption{Training observations grouped by available modality subset.
For PolyMNIST-D-Q and FFHQ64, categorical labels are observed in all groups.
$T$, $I$, and $A$ denote text, image, and audio, respectively.}
\label{tab:training-data}
\centering
\small
\setlength{\tabcolsep}{3pt}

\begin{minipage}[t]{0.32\linewidth}
\vspace{0pt}
\centering
\textbf{PolyMNIST} {\footnotesize $(+\,d,q)$}

\begin{tabular*}{\linewidth}{@{\extracolsep{\fill}}lrr@{}}
\toprule
Observed subset & Rows & \% \\
\midrule
$m_0,m_1,m_2$ & 55,000 & 25.0 \\
$m_0,m_1$     & 18,333 & 8.3 \\
$m_0,m_2$     & 18,334 & 8.3 \\
$m_1,m_2$     & 18,333 & 8.3 \\
$m_0$         & 36,667 & 16.7 \\
$m_1$         & 36,666 & 16.7 \\
$m_2$         & 36,667 & 16.7 \\
\midrule
Total & 220,000 & 100 \\
\bottomrule
\end{tabular*}
\end{minipage}
\hfill
\begin{minipage}[t]{0.32\linewidth}
\vspace{0pt}
\centering
\textbf{FFHQ64} {\footnotesize $(+$ Age, Gender$)$}

\begin{tabular*}{\linewidth}{@{\extracolsep{\fill}}lrr@{}}
\toprule
Observed subset & Rows & \% \\
\midrule
RGB, Seg., Norm. & 16,250 & 25.0 \\
RGB, Seg.        & 5,417  & 8.3 \\
RGB, Norm.       & 5,416  & 8.3 \\
Seg., Norm.      & 5,417  & 8.3 \\
RGB              & 10,833 & 16.7 \\
Seg.             & 10,834 & 16.7 \\
Norm.            & 10,833 & 16.7 \\
\midrule
Total & 65,000 & 100 \\
\bottomrule
\end{tabular*}
\end{minipage}
\hfill
\begin{minipage}[t]{0.32\linewidth}
\vspace{0pt}
\centering
\textbf{Image--Text--Audio}

\begin{tabular*}{\linewidth}{@{\extracolsep{\fill}}lrr@{}}
\toprule
Observed subset & Rows & \% \\
\midrule
$T,I$   & 272,703 & 37.8 \\
$I,A$   & 183,449 & 25.5 \\
$T,A$   & 91,028  & 12.6 \\
$T,I,A$ & 173,576 & 24.1 \\
\phantom{$T,I,A$} & & \\
\phantom{$T,I,A$} & & \\
\phantom{$T,I,A$} & & \\
\midrule
Total & 720,756 & 100 \\
\bottomrule
\end{tabular*}
\end{minipage}

\end{table}

PolyMNIST-D-Q, released by \citet{yeo2026muni} as PolyMNIST-Quadrant-Labels and built on PolyMNIST-Quadrant~\citep{zhang2026disentanglement}, contains three $64\times64$ RGB views with their digit (D) and quadrant (Q) labels. Each view composites a resized MNIST digit~\citep{lecun1998mnist} onto a modality-specific background by color inversion. The first 220,000 of the 230,000 tuples are used for training, and the last 10,000 are held out for evaluation. Because digits are drawn from both MNIST splits, the same digit image may appear in training and evaluation tuples.

FFHQ64 is constructed from FFHQ~\citep{karras2019stylegan} and contains RGB, 19-class segmentation, surface normals, age in ten bins (0--2, 3--6, 7--9, 10--14, 15--19, 20--29, 30--39, 40--49, 50--69, and 70+ years), and binary gender. Age, gender, and segmentation labels are taken from FFHQ-Aging~\citep{orel2020lifespan}, and surface normals are pseudo-labels from Marigold E2E-FT~\citep{martingarcia2025e2eft}. The first 65,000 aligned faces are used for training, and the last 5,000 are held out for evaluation.

The image--text--audio corpus (Table~\ref{tab:app-media-data}) contains 720,756 paired or triplet observations from the LAION-COCO aesthetic subset~\citep{schuhmann2022laioncoco}, VGGSound~\citep{chen2020vggsound}, AudioCaps~\citep{kim2019audiocaps}, Flickr30k~\citep{young2014flickr30k}, and E-MM1~\citep{broadbent2025ebind}. The observed modalities of each example follow the original pairing of its corpus rather than synthetic masking. All models operate on fixed, normalized 768-dimensional features from EmbeddingGemma-300M~\citep{schechtervera2025embeddinggemma} for text, CLIP ViT-L/14~\citep{radford2021clip} for images, and LAION CLAP HTSAT-base~\citep{wu2023laionclap} for audio, whose features are mapped to 768 dimensions by the audio autoencoder released with FlowBind~\citep{cha2026flowbind}. For the three-modality observations, we use the caption--image--audio triplets of E-MM1 rated Good Match in its human-annotated 1M split, excluding any triplet that overlaps our evaluation sets or the other training corpora.
\begin{table}[t]
\begin{minipage}[t]{0.49\linewidth}
\vspace{0pt}
\setlength{\abovecaptionskip}{0pt}
\setlength{\belowcaptionskip}{3pt}
\captionof{table}{Image--text--audio training data.}
\label{tab:app-media-data}\centering\small
\setlength{\tabcolsep}{3pt}
\begin{tabularx}{\linewidth}{@{}Ycr@{}}\toprule
Corpus & Modalities & Rows\\\midrule
LAION-COCO aesthetic subset & $T,I$ & 241,873\\
VGGSound & $I,A$ & 183,449\\
AudioCaps & $T,A$ & 91,028\\
Flickr30k & $T,I$ & 30,830\\
E-MM1 & $T,I,A$ & 173,576\\\midrule
Total & & 720,756\\\bottomrule
\end{tabularx}

\end{minipage}\hfill
\begin{minipage}[t]{0.49\linewidth}
\vspace{0pt}
\setlength{\abovecaptionskip}{0pt}
\setlength{\belowcaptionskip}{3pt}
\captionof{table}{Image--text--audio evaluation sets.}
\label{tab:app-media-pools}\centering\small
\setlength{\tabcolsep}{3pt}
\begin{tabularx}{\linewidth}{@{}lYr@{}}\toprule
Routes & Population & Count\\\midrule
$T\leftrightarrow I$ & COCO & 30,000\\
$T\leftrightarrow A$ & AudioCaps & 975\\
$I\leftrightarrow A$ & VGGSound & 15,446\\
$2\to1$, $1\to2$ & AudioCaps-based triplets & 975\\
Unconditional & Joint samples & 1,000\\\bottomrule
\end{tabularx}

\end{minipage}
\end{table}

\subsection{Model sizes and training budgets}
\label{app:baseline-controls}
In Table~\ref{tab:app-capacity}, CFM and DFM approximately match the trainable parameters of MUNITE, whereas MUNI and FlowBind keep the larger architectures of their released implementations.
\begin{table}[t]
\caption{Trainable parameters in millions, including learned encoders, priors, and decoders. Frozen foundation models and the frozen VQ-VAEs of DFM, which have 13.27M parameters on PolyMNIST-D-Q and FFHQ64, are excluded.}
\label{tab:app-capacity}\centering\small
\setlength{\tabcolsep}{4pt}
\begin{tabularx}{\linewidth}{@{}Yrrrrr@{}}\toprule
Task & FlowBind & MUNI & CFM & DFM & MUNITE\\\midrule
PolyMNIST-D-Q & \textemdash & 32.09 & 20.09 & 20.57 & 19.33\\
FFHQ64 & \textemdash & 130.99 & 105.11 & 107.27 & 106.64\\
Image--text--audio & 567.97 & 328.90 & 281.21 & 281.16 & 278.72\\\bottomrule
\end{tabularx}
\end{table}

All methods are trained with Adam~\citep{kingma2015adam} under the training budgets in Table~\ref{tab:app-optimization}. MUNI, MUNITE, CFM, and DFM are evaluated with EMA weights, with decay 0.999 on PolyMNIST-D-Q and 0.9995 on FFHQ64 and image--text--audio, whereas FlowBind follows its official pipeline without EMA.
\begin{table}[t]
\caption{Training budgets. Batch sizes are global, and all methods use a constant learning rate. FlowBind's epoch count assumes 702 updates per epoch.}
\label{tab:app-optimization}\centering\small
\setlength{\tabcolsep}{4pt}
\begin{tabularx}{\linewidth}{@{}Yrrrl@{}}\toprule
Method & Epochs & Updates & Batch & LR\\\midrule
\multicolumn{5}{@{}l}{\textit{PolyMNIST-D-Q}}\\
MUNI, MUNITE & 100 & 85,600 & 256 & $1.5\times10^{-4}$\\
CFM, DFM & 200 & 171,200 & 256 & $1.5\times10^{-4}$\\\midrule
\multicolumn{5}{@{}l}{\textit{FFHQ64}}\\
MUNI, MUNITE & 500 & 252,000 & 128 & $1.5\times10^{-4}$\\
CFM, DFM & 1,000 & 504,000 & 128 & $1.5\times10^{-4}$\\\midrule
\multicolumn{5}{@{}l}{\textit{Image--text--audio}}\\
MUNITE & 100 & 70,200 & 1,024 & $4.0\times10^{-4}$\\
MUNI, CFM, DFM & 200 & 140,400 & 1,024 & $4.0\times10^{-4}$\\
FlowBind & $\approx285$ & 200,000 & 1,024 & $1.0\times10^{-4}$\\\bottomrule
\end{tabularx}
\end{table}

\subsection{Inference budgets}
\label{app:sampling}
In Table~\ref{tab:app-sampling}, CFM is given as many network evaluations as the latent flow and decoder of MUNITE combined, counting two forward passes per step with classifier-free guidance~\citep{ho2022cfg}. DFM samples one token at a time in random order, which avoids the factorization error of predicting several tokens in parallel. FlowBind~\citep{cha2026flowbind} infers a latent from each source and averages them when several sources are given; it has no unconditional sampler.
\begin{table}[t]
\caption{Inference and decoding methods. Euler step counts equal the number of function evaluations, CFG denotes classifier-free guidance, and AR denotes autoregressive sampling. Rendering of image--text--audio features is not included.}
\label{tab:app-sampling}\centering\small
\setlength{\tabcolsep}{4pt}
\begin{tabularx}{\linewidth}{@{}lYYY@{}}\toprule
Method & Conditional & Unconditional & Decoding\\\midrule
\multicolumn{4}{@{}l}{\textit{PolyMNIST-D-Q}}\\
MUNITE & Euler 50 & Euler 50 & Image Euler 50, CFG 1.5\\
MUNI & Posterior sampling & Euler 50 & Image Euler 50, CFG 1.5\\
CFM & Euler 150 & Euler 150 & \textemdash\\
DFM & Random-order AR & Random-order AR & VQ decoder\\\midrule
\multicolumn{4}{@{}l}{\textit{FFHQ64}}\\
MUNITE & Euler 50 & Euler 50 & Image Euler 50, CFG 1.5\\
MUNI & Posterior sampling & Euler 50 & Image Euler 50, CFG 1.5\\
CFM & Euler 150 & Euler 150 & \textemdash\\
DFM & Random-order AR & Random-order AR & VQ decoder\\\midrule
\multicolumn{4}{@{}l}{\textit{Image--text--audio}}\\
MUNITE & Euler 25 & Euler 25 & Image-feature Euler 25, CFG 4\\
MUNI & Posterior sampling & Euler 25 & Image-feature Euler 25, CFG 4\\
CFM & Euler 75 & Euler 75 & \textemdash\\
DFM & Random-order AR & Random-order AR & VQ decoder\\
FlowBind & Euler 10 per source & \textemdash & Euler 10 per target\\\bottomrule
\end{tabularx}
\end{table}

For the models trained on our benchmark, generated image and audio features are rendered by Stable unCLIP small~\citep{stabilityai2023stableunclip}, an unCLIP-style decoder~\citep{ramesh2022unclip}, and by AudioLDM-m-full~\citep{liu2023audioldm}, respectively, with 20 denoising steps as in the benchmark. Because CFM and DFM have no modality-specific decoders, they generate normalized text features as FlowBind does, and all three decode them with the frozen FlowBind text decoder, on which the text decoders of MUNI and MUNITE are built (Appendix~\ref{app:architecture}).

\subsection{Verifiers}
\label{app:evaluators}
Following MUNI~\citep{yeo2026muni}, accuracy and coherence on PolyMNIST-D-Q and FFHQ64 are measured with verifiers that predict the labels and dense maps of generated images. Each verifier fine-tunes a pretrained SegFormer backbone~\citep{xie2021segformer}, MiT-B0 for PolyMNIST-D-Q and MiT-B5 for FFHQ64, on the classification, segmentation, and normal-estimation tasks of its benchmark. The verifiers are trained on all 230,000 PolyMNIST-D-Q tuples and all 70,000 FFHQ faces, including the held-out ones, and are kept frozen when scoring all methods.

\section{Evaluation metrics and additional results}
\label{app:metrics}

\subsection{PolyMNIST-D-Q}
\label{app:poly-results}
PolyMNIST-D-Q is evaluated on the 10,000 held-out tuples and on 10,000 unconditional samples. FD uses verifier features instead of Inception features and is averaged over the three views.

For label-to-image generation, accuracy compares the verifier's prediction for each generated view with the input label, and the $(d,q)\to m_i$ route counts a view as correct only if both labels match. Unconditional coherence is the fraction of samples in which the verifier assigns the generated label to all three views. One-to-many coherence is the agreement between pairs of generated views on the attribute that the input leaves out, the quadrant for $d\to(m_0,m_1,m_2)$ and the digit for $q\to(m_0,m_1,m_2)$, averaged over the three pairs and all samples.

Pairwise agreement alone does not show that the missing attribute is sampled correctly, since a model that always generates the same quadrant or digit for a given input would also agree fully. Table~\ref{tab:app-poly-missing} therefore also compares the missing label of each generated view with its data conditional, which is uniform because every digit--quadrant pair occurs equally often. MUNITE stays close to the finite-sample reference for both attributes, so its agreement does not come from a collapsed prediction. For the missing quadrant, MUNI is equally close to uniform, yet its views agree only at chance level: the quadrant of each view is correctly distributed, but the three views choose their quadrants independently. When $m_1$ and $m_2$ are decoded from separately drawn latents (Indep.), MUNITE's agreement also falls to chance, while the accuracy of each view on the input label remains 0.995 and 1.000 for digit and quadrant inputs. The agreement among views is therefore carried by the shared latent sample.
\begin{table}[b]
\caption{Missing attribute in PolyMNIST-D-Q one-to-many generation. TV is the total variation distance between the missing-label distribution of a generated view and the uniform data conditional, averaged over input labels and views; Ref.\ is its expected value for the same number of samples drawn from that conditional. Indep.\ is MUNITE with $m_1$ and $m_2$ decoded from separately drawn latents, keeping $m_0$ and all decoder noise fixed. Chance is the coherence of independent uniform views.}
\label{tab:app-poly-missing}\centering\small
\setlength{\tabcolsep}{4pt}
\begin{tabular}{@{}lccccccc@{}}\toprule
& \multicolumn{3}{c}{TV $\downarrow$} & \multicolumn{4}{c}{Coherence $\uparrow$}\\\cmidrule(lr){2-4}\cmidrule(l){5-8}
Route & MUNI & MUNITE & Ref. & MUNI & MUNITE & Indep. & Chance\\\midrule
$d\to(m_0,m_1,m_2)$ & 0.029 & 0.029 & 0.022 & 0.2528 & 1.0000 & 0.2500 & 0.25\\
$q\to(m_0,m_1,m_2)$ & 0.109 & 0.034 & 0.024 & 0.1074 & 0.9904 & 0.1014 & 0.10\\\bottomrule
\end{tabular}
\end{table}

\subsection{FFHQ64}
\label{app:ffhq-results}
FFHQ64 is evaluated on the 5,000 held-out tuples and on 5,000 unconditional samples, and RGB fidelity is measured by FID~\citep{heusel2017fid} with InceptionV3~\citep{szegedy2016inception} features. Age and gender accuracy compare the verifier's predictions for a generated face with the input labels, requiring both to match on the two-label route, and unconditional age and gender coherence compare them with the generated labels. Generated dense maps are scored against the verifier's predictions for the generated face. RGB--segmentation coherence is the pixel accuracy over the 19 classes, and normal error is one minus the mean per-pixel cosine similarity. For $(\mathrm{Age},\mathrm{Gender})\to(\mathrm{RGB},\mathrm{Seg.},\mathrm{Normals})$, the inputs provide no reference maps, so these scores compare the generated targets with one another.

\subsection{Image--text--audio}
\label{app:media-results}
The evaluation follows the routes and evaluation sets of FlowBind and MUNI~\citep{cha2026flowbind,yeo2026muni}, listed in Table~\ref{tab:app-media-pools}. The 975 joint-evaluation triplets contain AudioCaps~\citep{kim2019audiocaps} captions and real audio with FLUX.1-generated images~\citep{blackforestlabs2024flux}, so the three modalities are not co-recorded. CIDEr is computed with one reference caption per COCO~\citep{lin2014coco} or AudioCaps item.

Fidelity is measured by FID~\citep{heusel2017fid} with InceptionV3~\citep{szegedy2016inception} features, FAD~\citep{kilgour2019fad} with VGGish~\citep{hershey2017vggish} embeddings, and CIDEr~\citep{vedantam2015cider}. Source--target alignment is measured by CLIP ViT-B/32~\citep{radford2021clip} for text--image and LAION larger-CLAP~\citep{wu2023laionclap} for text--audio, both different from the encoders used to extract the training features. CLIP and CLAP similarities and AIS are reported on a scale of 100.

AIS measures how highly a generated sample ranks its paired counterpart within a reference pool, using Wav2CLIP~\citep{wu2022wav2clip} audio and CLIP ViT-B/32~\citep{radford2021clip} image embeddings as in FlowBind and MUNI~\citep{cha2026flowbind,yeo2026muni}. For $N$ generated samples with normalized embeddings $g_i$, the normalized embeddings $c_i$ of their counterparts, and a pool $R$ of counterpart embeddings spanning the evaluation set,
\begin{equation}
\mathrm{AIS}=\frac{100}{N|R|}\sum_{i=1}^{N}\sum_{r\in R}\mathbf 1[\langle g_i,r\rangle<\langle g_i,c_i\rangle].
\label{eq:ais}
\end{equation}
For text-conditioned joint image--audio generation, each generated image is paired with its generated audio, and the pool contains the generated audios.

\paragraph{External pretrained models.}
\label{app:external-baselines}
CoDi~\citep{tang2023codi} and OmniFlow~\citep{li2025omniflow} are evaluated as released pretrained models. We measure their one-to-many coherence on the same 975 triplets and take their other scores from MUNI~\citep{yeo2026muni}, where they are reported to two decimal places. For text-conditioned joint image--audio generation, CoDi uses its default DDIM sampler~\citep{song2021ddim} with 50 steps and guidance 7.5, and OmniFlow uses 28 steps with guidance 5.

\subsubsection{Single-target generation}
\label{app:media-single-results}
Tables~\ref{tab:media-fidelity}--\ref{tab:media-many} report the one-to-one and many-to-one scores summarized in Section~\ref{sec:exp-media}. MUNITE has the best score on three of the six one-to-one fidelity measures (Table~\ref{tab:media-fidelity}). Relative to MUNI, it has higher text--audio and lower text--image alignment in both one-to-one and many-to-one generation (Tables~\ref{tab:media-alignment} and~\ref{tab:media-many}). Since its single outputs are not uniformly better, the coherence gains in Table~\ref{tab:media-coherence} are not explained by single-output quality.
\begin{table}[htbp]
\caption{Image--text--audio one-to-one fidelity on the evaluation sets of Table~\ref{tab:app-media-pools}, grouped by target modality. CoDi and OmniFlow scores are taken from \citet[Table~8]{yeo2026muni}. Bold marks the best score in each column.}
\label{tab:media-single}\label{tab:media-fidelity}\centering\small
\setlength{\tabcolsep}{4pt}
\begin{tabularx}{\linewidth}{@{}l*{6}{>{\centering\arraybackslash}X}@{}}\toprule
& \multicolumn{2}{c}{Image: FID $\downarrow$} & \multicolumn{2}{c}{Audio: FAD $\downarrow$} & \multicolumn{2}{c}{Text: CIDEr $\uparrow$}\\
\cmidrule(lr){2-3}\cmidrule(lr){4-5}\cmidrule(l){6-7}
Method & $T\to I$ & $A\to I$ & $T\to A$ & $I\to A$ & $I\to T$ & $A\to T$\\\midrule
CoDi & 24.97 & 53.60 & 8.67 & 13.92 & 11.97 & 7.70\\
OmniFlow & 20.52 & 97.56 & 4.08 & 4.87 & 26.60 & 31.42\\
\midrule
MUNI & 15.615 & 22.786 & 3.507 & 2.246 & \textbf{44.400} & 51.794\\
CFM & \textbf{13.455} & 28.971 & 3.564 & 2.105 & 13.085 & 19.729\\
DFM & 14.400 & 28.374 & 3.670 & \textbf{1.991} & 11.263 & 20.909\\
FlowBind & 14.826 & 24.270 & 3.477 & 2.166 & 17.211 & 30.905\\
MUNITE & 15.939 & \textbf{21.511} & \textbf{3.059} & 2.135 & 37.409 & \textbf{57.592}\\
\bottomrule\end{tabularx}
\end{table}

\begin{table}[htbp]
\caption{Image--text--audio one-to-one alignment. CoDi and OmniFlow scores are taken from \citet[Table~9]{yeo2026muni}. Higher is better, and bold marks the best score in each column.}
\label{tab:media-alignment}\centering\small
\setlength{\tabcolsep}{4pt}
\begin{tabularx}{\linewidth}{@{}l*{6}{>{\centering\arraybackslash}X}@{}}\toprule
& \multicolumn{2}{c}{Text--image: CLIP} & \multicolumn{2}{c}{Text--audio: CLAP} & \multicolumn{2}{c}{Image--audio: AIS}\\
\cmidrule(lr){2-3}\cmidrule(lr){4-5}\cmidrule(l){6-7}
Method & $T\to I$ & $I\to T$ & $T\to A$ & $A\to T$ & $I\to A$ & $A\to I$\\\midrule
CoDi & 29.70 & 26.05 & 16.72 & 32.72 & 59.45 & 84.47\\
OmniFlow & \textbf{30.76} & 27.48 & 30.07 & \textbf{44.74} & 76.52 & 64.22\\
\midrule
MUNI & 29.718 & \textbf{28.373} & 30.596 & 40.102 & 90.390 & \textbf{94.773}\\
CFM & 27.710 & 24.918 & 25.127 & 32.094 & 82.888 & 80.842\\
DFM & 27.321 & 24.353 & 24.883 & 32.262 & 79.903 & 79.933\\
FlowBind & 29.202 & 26.004 & 28.011 & 35.393 & 87.621 & 92.445\\
MUNITE & 29.007 & 27.526 & \textbf{30.695} & 40.335 & \textbf{90.655} & 94.254\\
\bottomrule\end{tabularx}
\end{table}

\begin{table}[htbp]
\caption{Image--text--audio many-to-one alignment on 975 triplets, with each generated target scored against both source modalities. CoDi and OmniFlow scores are taken from \citet[Table~3]{yeo2026muni}. Higher is better, and bold marks the best score in each column.}
\label{tab:media-many}\centering\small
\setlength{\tabcolsep}{4pt}
\begin{tabularx}{\linewidth}{@{}l*{6}{>{\centering\arraybackslash}X}@{}}\toprule
& \multicolumn{2}{c}{$(I,A)\to T$} & \multicolumn{2}{c}{$(T,A)\to I$} & \multicolumn{2}{c}{$(T,I)\to A$}\\
\cmidrule(lr){2-3}\cmidrule(lr){4-5}\cmidrule(l){6-7}
Method & CLIP & CLAP & CLIP & AIS & CLAP & AIS\\\midrule
CoDi & 24.05 & 33.72 & 24.98 & 85.52 & 11.06 & 65.31\\
OmniFlow & 24.73 & 36.26 & 26.41 & 81.51 & 13.50 & 63.55\\
\midrule
MUNI & \textbf{27.922} & 36.695 & \textbf{27.056} & 85.168 & 30.524 & 82.621\\
CFM & 25.221 & 31.889 & 25.712 & 80.305 & 24.907 & 76.098\\
DFM & 24.855 & 33.259 & 25.398 & 80.347 & 25.141 & 75.846\\
FlowBind & 25.541 & 31.624 & 25.841 & \textbf{88.089} & 28.748 & 81.878\\
MUNITE & 27.289 & \textbf{38.969} & 26.721 & 86.745 & \textbf{30.653} & \textbf{82.677}\\
\bottomrule\end{tabularx}
\end{table}

\subsubsection{Component ablations}
\label{app:media-ablations}
Table~\ref{tab:media-ablation} ablates the training components of MUNITE on image--text--audio. \emph{No target detach} lets reconstruction gradients flow through the target's key/value paths, \emph{no contrastive} removes $\mathcal L_{\mathrm{con}}$, and \emph{full-only distillation} applies the distillation loss only to fully observed examples, so partially observed examples no longer supervise the latent conditionals.

Every variant keeps one-to-many coherence within two points of the full model and above MUNI on all three routes (Table~\ref{tab:media-coherence}), so the one-to-many advantage does not depend on any single auxiliary component. Unconditional coherence depends more on these components. Target detaching matters most: without it, CLAP and AIS coherence drop by 5.6 and 9.6 points, and all three unconditional scores fall below those of MUNI. Restricting distillation to fully observed examples lowers CLAP and AIS coherence by 2.8 and 4.6 points, since partially observed examples then no longer shape the latent distribution sampled in unconditional generation.

For single-target generation, removing target detaching or the contrastive loss degrades fidelity on all six one-to-one routes, most visibly image-to-audio FAD and audio-to-text CIDEr without the contrastive loss. Full-only distillation gives mixed results, slightly improving fidelity on the four routes that generate or condition on images and degrading text-to-audio FAD and audio-to-text CIDEr, so partial-data distillation mainly benefits coherence.
\begin{table}[t]
\caption{Component ablations on image--text--audio. Control is the full MUNITE model of Tables~\ref{tab:media-coherence} and~\ref{tab:media-fidelity}--\ref{tab:media-many}. Each variant removes one component and otherwise keeps the control's architecture, optimization, and evaluation protocol. The $1\to N$ and unconditional rows report coherence among jointly generated modalities. Best results in each row are in bold.}
\label{tab:media-ablation}\centering\small
\setlength{\tabcolsep}{4pt}
\begin{tabularx}{\linewidth}{@{}llYrrrr@{}}\toprule
Route & Input $\to$ output & Metric & Control & \begin{tabular}[b]{@{}r@{}}No target\\detach\end{tabular} & \begin{tabular}[b]{@{}r@{}}No\\contrastive\end{tabular} & \begin{tabular}[b]{@{}r@{}}Full-only\\distillation\end{tabular}\\\midrule
\multirow{12}{*}{$1\to1$} & $T\to I$ & FID $\downarrow$ & 15.939 & 16.700 & 17.308 & \textbf{15.230}\\
 &  & CLIP $\uparrow$ & 29.007 & 28.543 & 27.750 & \textbf{29.593}\\
 & $I\to T$ & CIDEr $\uparrow$ & 37.409 & 35.414 & 35.741 & \textbf{39.667}\\
 &  & CLIP $\uparrow$ & 27.526 & 27.261 & 26.919 & \textbf{27.882}\\
 & $T\to A$ & FAD $\downarrow$ & \textbf{3.059} & 3.145 & 3.321 & 3.234\\
 &  & CLAP $\uparrow$ & 30.695 & 30.458 & \textbf{31.461} & 30.540\\
 & $A\to T$ & CIDEr $\uparrow$ & \textbf{57.592} & 53.462 & 51.927 & 53.302\\
 &  & CLAP $\uparrow$ & 40.335 & 40.670 & 39.588 & \textbf{40.678}\\
 & $I\to A$ & FAD $\downarrow$ & 2.135 & 2.146 & 2.622 & \textbf{2.094}\\
 &  & AIS $\uparrow$ & 90.655 & 90.768 & \textbf{91.952} & 90.157\\
 & $A\to I$ & FID $\downarrow$ & 21.511 & 21.923 & 22.466 & \textbf{20.823}\\
 &  & AIS $\uparrow$ & 94.254 & 93.932 & 93.565 & \textbf{94.260}\\\midrule
\multirow{6}{*}{$N\to1$} & $(I,A)\to T$ & CLIP $\uparrow$ & 27.289 & 27.191 & 27.148 & \textbf{27.441}\\
 &  & CLAP $\uparrow$ & \textbf{38.969} & 38.400 & 36.322 & 38.517\\
 & $(T,A)\to I$ & CLIP $\uparrow$ & 26.721 & 26.670 & 26.581 & \textbf{26.799}\\
 &  & AIS $\uparrow$ & \textbf{86.745} & 85.898 & 85.624 & 85.127\\
 & $(T,I)\to A$ & CLAP $\uparrow$ & 30.653 & 30.511 & \textbf{31.420} & 30.804\\
 &  & AIS $\uparrow$ & 82.677 & 82.461 & \textbf{83.407} & 83.055\\\midrule
\multirow{3}{*}{$1\to N$} & $T\to(I,A)$ & AIS $\uparrow$ & 83.336 & \textbf{84.174} & 83.771 & 82.487\\
 & $I\to(T,A)$ & CLAP $\uparrow$ & 34.762 & \textbf{34.844} & 32.852 & 34.392\\
 & $A\to(T,I)$ & CLIP $\uparrow$ & \textbf{26.334} & 26.261 & 26.024 & 26.287\\\midrule
\multirow{3}{*}{Uncond.} & All three & CLIP $(T,I)$ $\uparrow$ & 27.398 & 26.713 & 26.639 & \textbf{27.539}\\
 &  & CLAP $(T,A)$ $\uparrow$ & \textbf{30.220} & 24.657 & 29.370 & 27.379\\
 &  & AIS $(I,A)$ $\uparrow$ & \textbf{87.045} & 77.451 & 85.760 & 82.429\\\bottomrule
\end{tabularx}
\end{table}

\FloatBarrier

\FloatBarrier
\renewcommand{\bottomfraction}{.7}
\section{Qualitative results}
\label{app:qualitative}
\suppressfloats[t]
Qualitative examples complement the quantitative results of Section~\ref{sec:experiments}, particularly for joint generation, where they show whether the generated outputs agree on content that the input leaves unspecified. The examples are selected to illustrate each route, and the outputs of routes whose only output is audio are provided on the project page.

\subsection{PolyMNIST-D-Q}
\label{app:qualitative-poly}
\begin{figure}[t]
\centering
\includegraphics[width=\linewidth]{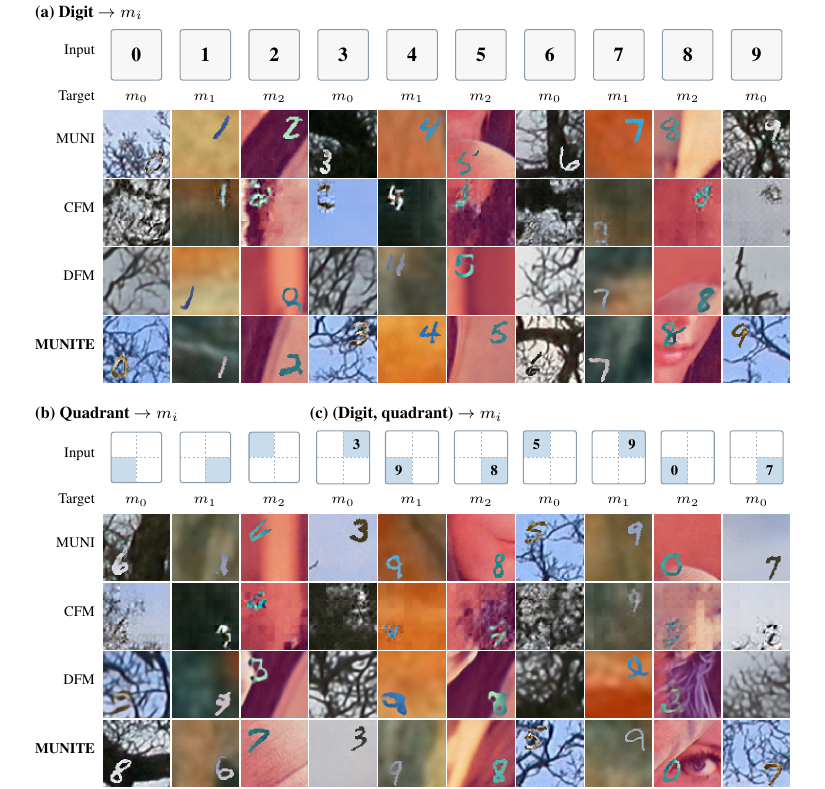}
\caption{PolyMNIST-D-Q single-target generation. Icons show the observed labels: a digit, the conditioned quadrant shaded in a $2\times2$ grid, or the digit placed in that quadrant. Each column generates the view listed under its input.}
\label{fig:qual-poly-single}
\end{figure}
\begin{figure}[p]
\centering
\includegraphics[width=\linewidth]{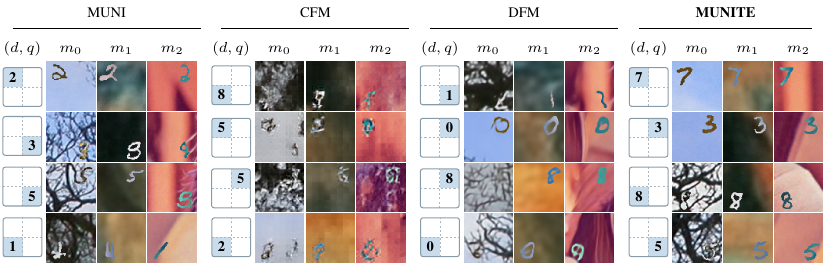}
\caption{Unconditional PolyMNIST-D-Q co-generation. Each group is one joint sample of the three views and both labels, and its icon shows the generated digit in the generated quadrant.}
\label{fig:qual-poly-uncond}
\end{figure}
\begin{figure}[p]
\centering
\includegraphics[width=\linewidth]{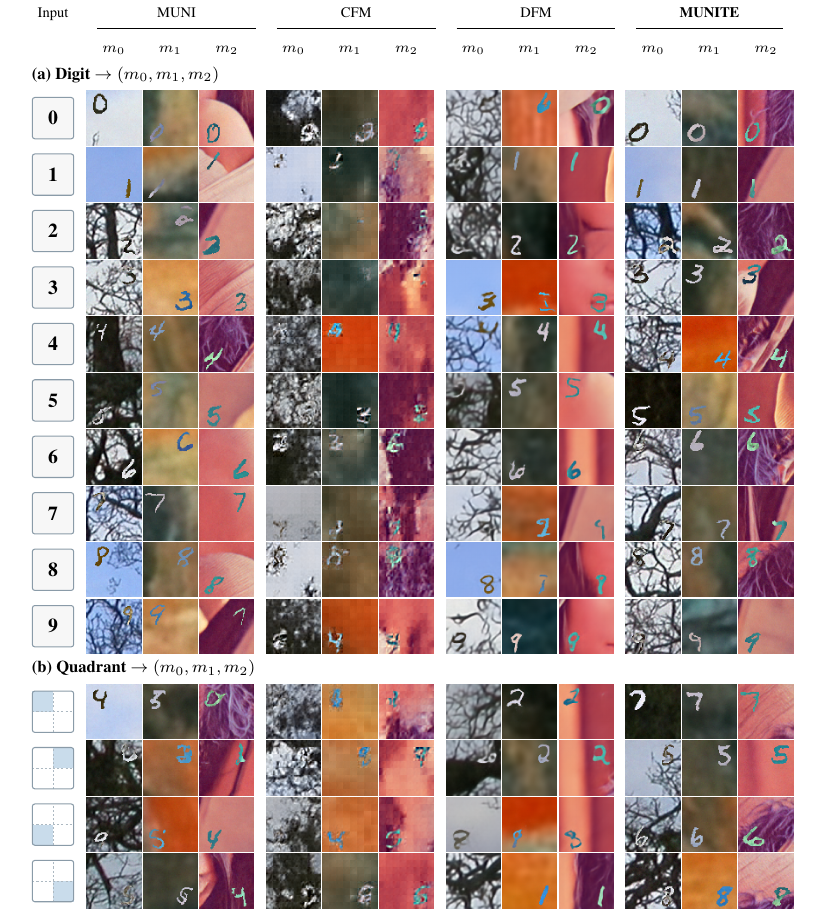}
\caption{PolyMNIST-D-Q one-to-many generation. Each row generates all three views from a digit (a) or a quadrant (b), so the views must agree on the unobserved attribute.}
\label{fig:qual-poly-1toN}
\end{figure}
In single-target generation (Figure~\ref{fig:qual-poly-single}), MUNI and MUNITE render the conditioned labels in every example, whereas CFM often produces illegible digits and DFM frequently omits the digit or renders a different one, in line with their accuracies in Table~\ref{tab:image-main}.

In unconditional co-generation (Figure~\ref{fig:qual-poly-uncond}), the three views and two labels of a sample must agree with one another. All four samples of MUNITE agree, whereas MUNI's views sometimes disagree, showing different digits or placing one digit in different quadrants.

In the one-to-many routes (Figure~\ref{fig:qual-poly-1toN}), the views must also agree on the attribute that the input leaves unspecified. MUNITE places a given digit in the same quadrant in all three views and draws the same digit for a given quadrant. MUNI's views are each consistent with the input but draw the missing attribute separately, so the digit 1 appears in three different quadrants and the lower-left quadrant holds the digits 0, 5, and 4. This is the gap between target-wise sufficiency and joint coherence discussed in Section~\ref{sec:relation-to-multimodal}.
\FloatBarrier

\subsection{FFHQ64}
\label{app:qualitative-ffhq}
\suppressfloats[t]
\begin{figure}[!h]
\centering
\includegraphics[width=\linewidth]{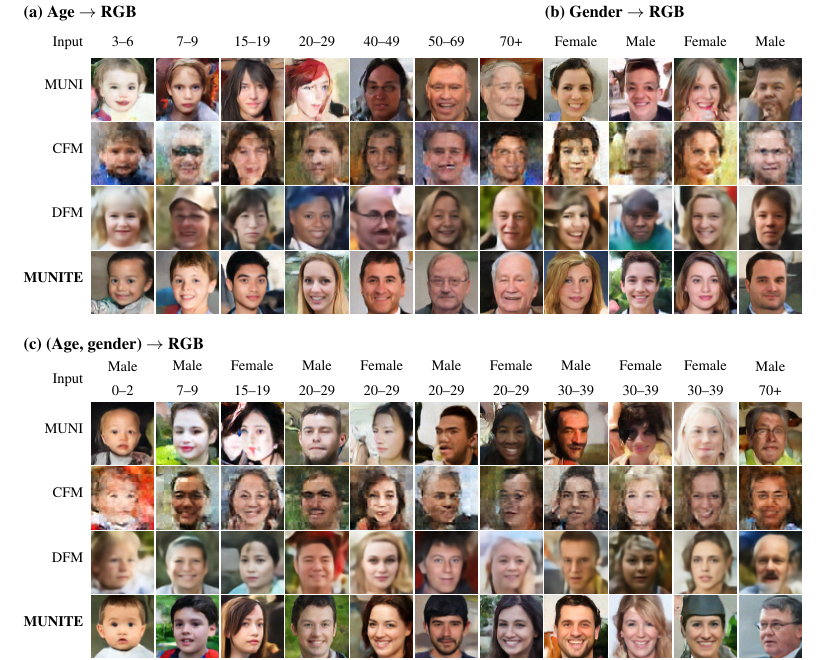}
\caption{FFHQ64 RGB generation from age (a), gender (b), or both attributes (c).}
\label{fig:qual-ffhq-single}
\end{figure}
\begin{figure}[tp]
\centering
\includegraphics[width=\linewidth]{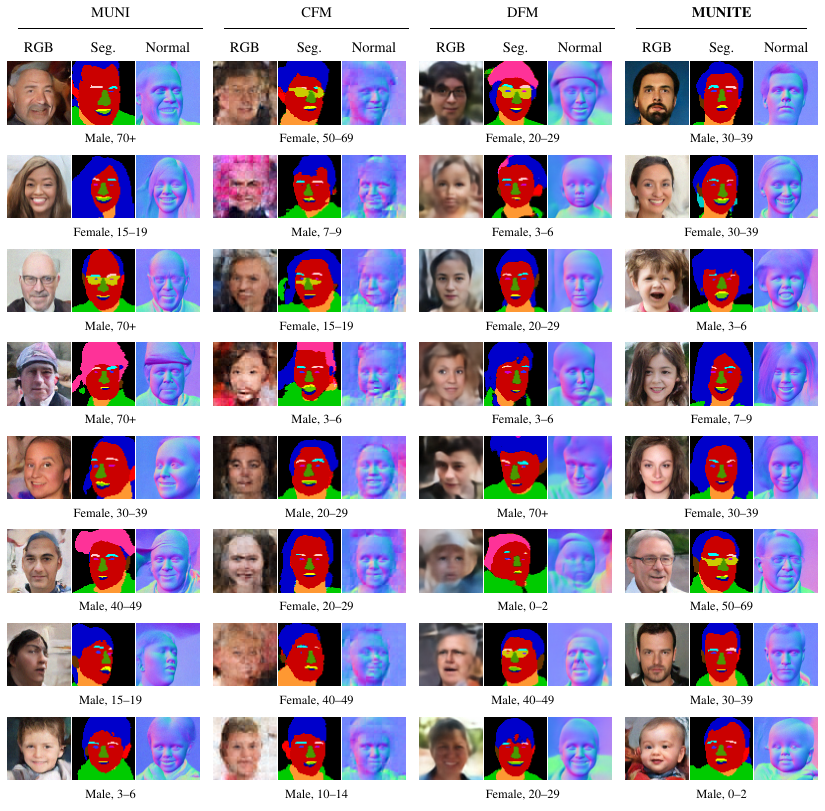}
\caption{Unconditional FFHQ64 co-generation of all five modalities. The generated gender and age labels of each joint sample appear below its maps.}
\label{fig:qual-ffhq-uncond}
\end{figure}
\begin{figure}[tp]
\centering
\includegraphics[width=\linewidth]{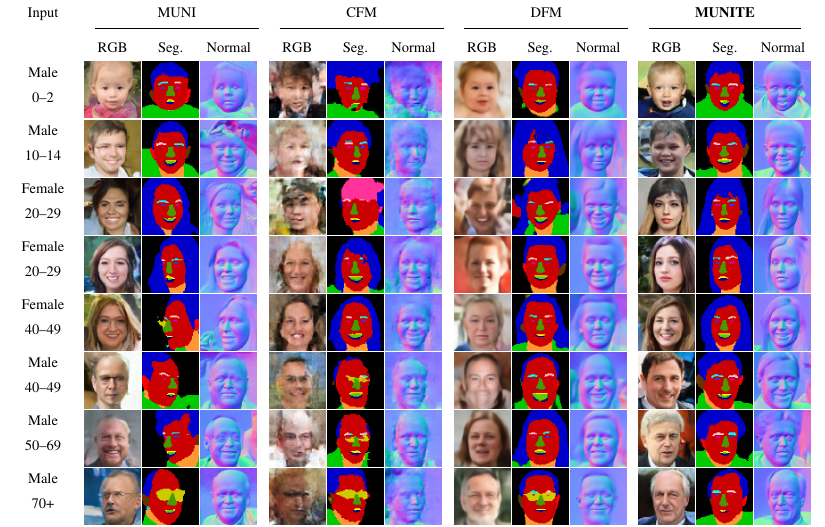}
\caption{FFHQ64 many-to-many generation of RGB, segmentation, and surface normals from age and gender. The three maps of each group come from one joint sample.}
\label{fig:qual-ffhq-mn}
\end{figure}
In RGB generation from age and gender (Figure~\ref{fig:qual-ffhq-single}), MUNITE's faces follow the conditioned age from young children to older adults and match the conditioned gender. MUNI's faces also follow the attributes but are occasionally distorted, CFM's show strong artifacts, and DFM's are noticeably blurred.

In unconditional co-generation of all five modalities (Figure~\ref{fig:qual-ffhq-uncond}), MUNITE's generated labels agree with the apparent age and gender of its faces, and its segmentation and normals align with the RGB image. The man in the 50--69 sample, for instance, wears glasses in the RGB image, and the segmentation marks them as well.

In the many-to-many route (Figure~\ref{fig:qual-ffhq-mn}), the segmentation and normals must describe the generated face. MUNITE's maps follow the head pose, hairline, and facial layout of its RGB image in every example. MUNI's maps sometimes describe a different face, as for the male 40--49 input, whose segmentation omits the glasses in the RGB image.
\FloatBarrier

\subsection{Image--text--audio}
\label{app:qualitative-media}
\suppressfloats[t]
\begin{figure}[tp]
\centering
\includegraphics[width=\linewidth]{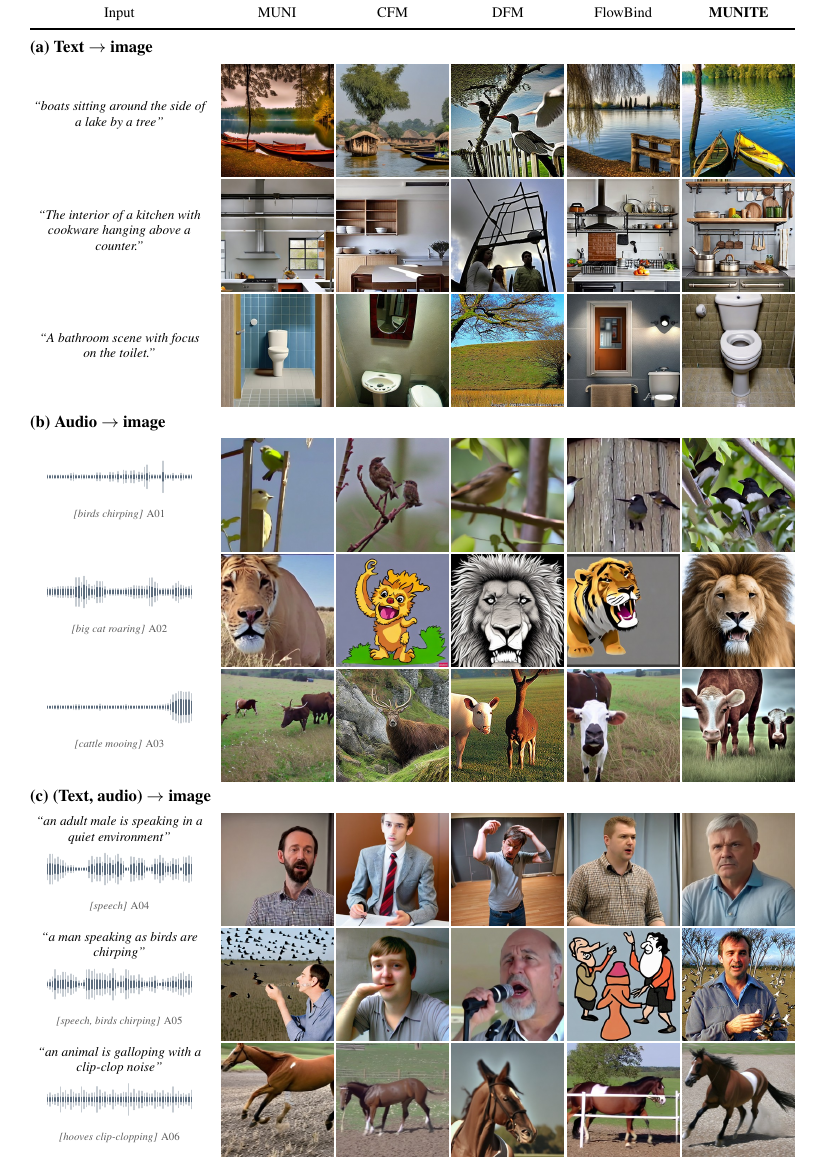}
\caption{Image--text--audio one-to-one and many-to-one generation of images. Audio inputs are shown as waveforms.}
\label{fig:qual-media-image}
\end{figure}
\begin{figure}[tp]
\centering
\includegraphics[width=\linewidth]{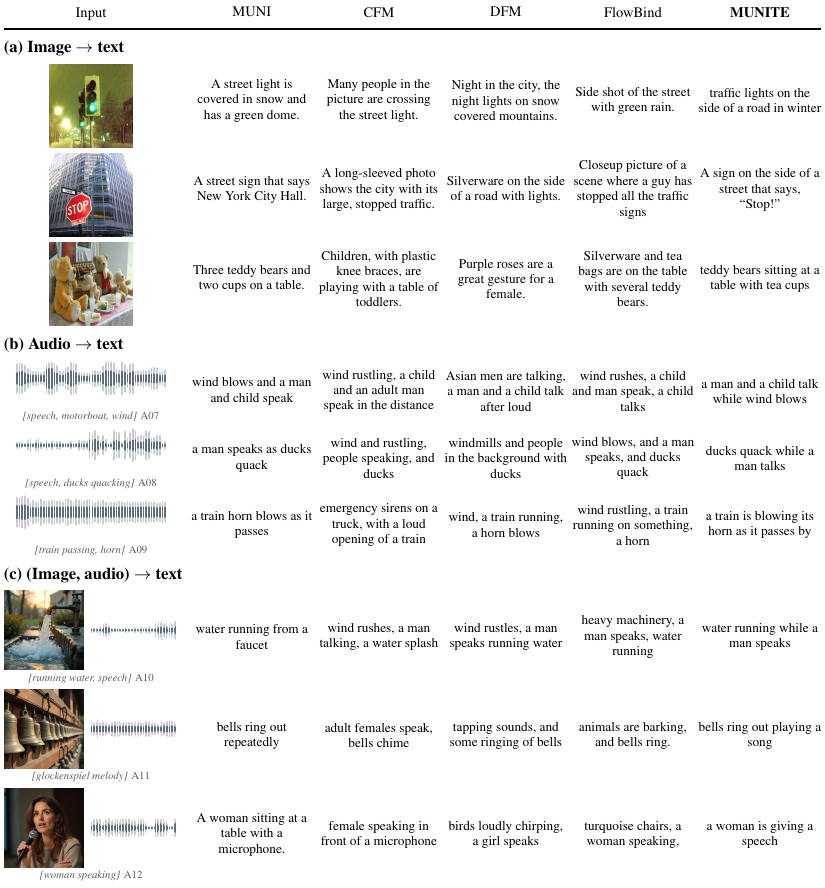}
\caption{Image--text--audio one-to-one and many-to-one generation of text. Audio inputs are shown as waveforms.}
\label{fig:qual-media-text}
\end{figure}
\begin{figure}[p]
\centering
\includegraphics[width=\linewidth]{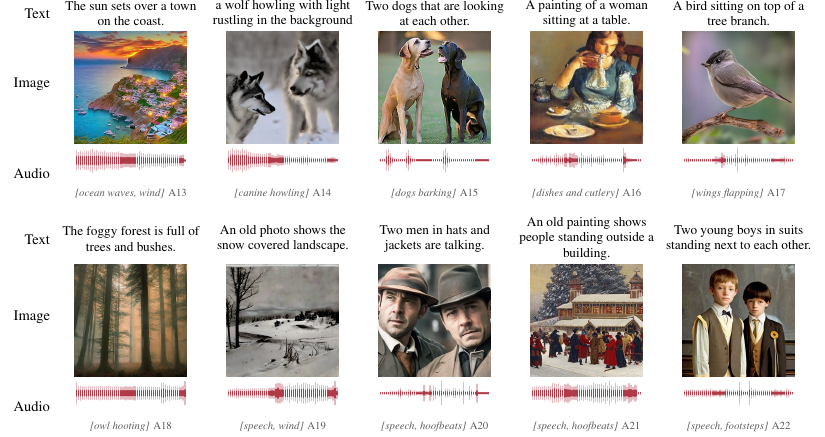}
\caption{Unconditional image--text--audio co-generation by MUNITE. Each column decodes one latent sample into text, image, and audio.}
\label{fig:qual-media-uncond}
\end{figure}
\begin{figure}[p]
\centering
\includegraphics[width=\linewidth]{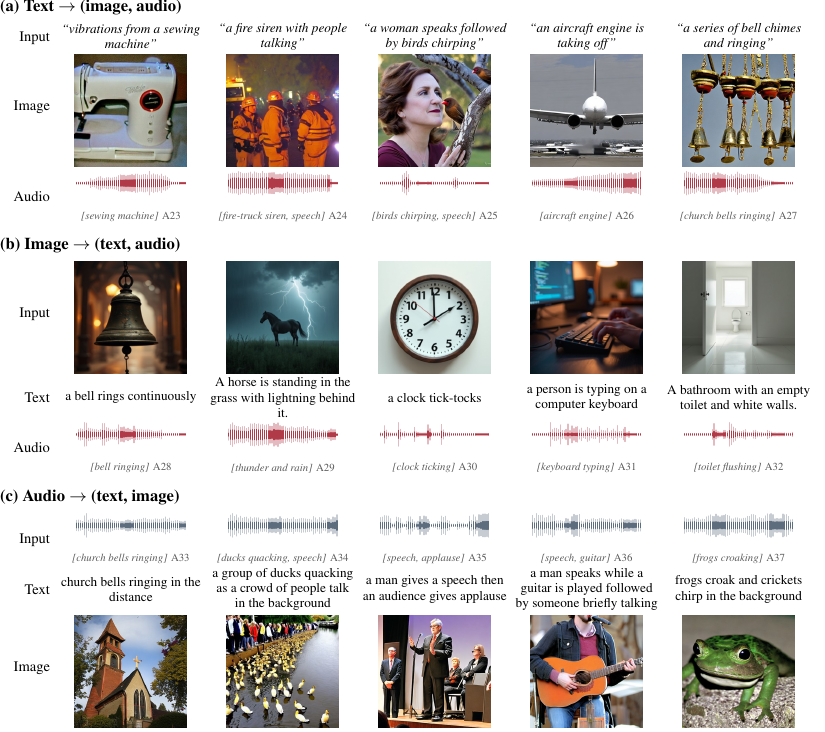}
\caption{One-to-many image--text--audio generation by MUNITE. In each column, the two outputs are decoded independently from one latent sample inferred from the input.}
\label{fig:qual-media-1toN}
\end{figure}
Audio clips are provided on the project page under the IDs printed below their waveforms. The bracketed descriptions are only a reading aid and summarize the top AudioSet labels predicted by PaSST~\citep{koutini22passt}.

In one-to-one and many-to-one generation (Figures~\ref{fig:qual-media-image} and~\ref{fig:qual-media-text}), MUNITE's outputs follow the specific content of their inputs. Its image for the roaring lions in A02 shows a lion where FlowBind shows a tiger, and its caption for the stop-sign image reads the sign. With two inputs, its captions draw on both. For the faucet image paired with running water and speech (A10), MUNITE mentions both the water and the speaker, whereas MUNI mentions only the water.

For joint generation, we show samples of MUNITE (Figures~\ref{fig:qual-media-uncond} and~\ref{fig:qual-media-1toN}). An unconditional sample has no input, so whatever its outputs share comes from the latent sample. The three outputs of each column describe one scene, such as a coastal town at sunset with the sound of waves (A13). For the painting of a woman at a table, the image shows cups and a plate and the audio contains dishes and cutlery (A16), a detail that the caption leaves out.

In the one-to-many routes, the two outputs are decoded independently from one latent sample, so any detail they share beyond the input must also be carried by that sample. For the ducks and speech in A34, both outputs turn the speech into a crowd, with the caption describing people talking near quacking ducks and the image showing a crowd lining a canal full of ducks. Given the horse under lightning, the caption mentions the lightning and the audio adds thunder and rain (A29).
\FloatBarrier

\end{document}